\documentclass[twoside]{article}

\usepackage[preprint]{aistats2027}
\usepackage{amsmath, amssymb, amsthm}
\usepackage[round]{natbib}
\usepackage{graphicx}
\usepackage{tikz}
\usetikzlibrary{automata, positioning, arrows.meta}
\usepackage{subcaption}
\usepackage{booktabs}
\usepackage{multirow}
\usepackage{lipsum}
\usepackage{comment}

\newcommand{\fracpartial}[2]{\frac{\partial #1}{\partial  #2}}

\newcommand{\attn}{\boldsymbol{\alpha}}
\newcommand{\readout}{\boldsymbol{\beta}}
\newcommand{\E}{\mathbb E}
\newcommand{\prob}{\mathbb P}
\newcommand{\given}{\;|\;}
\newcommand{\cp}[1]{\left( #1 \right)}
\newcommand{\cb}[1]{\left[ #1 \right]}
\newcommand{\good}[1]{\(\mathbf{#1}\)}
\newcommand{\hd}[2]{\shortstack{#1\\#2}}
\newcommand{\sd}[1]{{\scriptsize(\(#1\))}}
\definecolor{oao}{HTML}{CC79A7}
\definecolor{aoa}{HTML}{009E73}
\definecolor{a}{HTML}{56B4E9}
\definecolor{o}{HTML}{D55E00}
\theoremstyle{plain}
\newtheorem{theorem}{Theorem}
\newtheorem{lemma}{Lemma}
\newtheorem{proposition}{Proposition}
\newtheorem{corollary}{Corollary}

\theoremstyle{definition}

\usepackage{xcolor}

\begin{document}

%
\runningtitle{Attention Learning Can Lag Output Layer Learning}

%

\twocolumn[

    \aistatstitle{
    Training Dynamics of Transformers on Unbounded Inputs:
    \\Attention Learning Can Lag Output Layer Learning
    }
    
    \aistatsauthor{Yaomengxi Han, Debarghya Ghoshdastidar}
    
    \aistatsaddress{Technical University of Munich} 
]

\begin{abstract}
The success of transformers is attributed to several factors, including their capacity to express formal languages over variable-length inputs, and ability of the attention mechanism to focus on relevant input features. However, there is limited understanding of the training dynamics of shallow transformers, specifically how to decouple the learning of the attention and the output (feedforward) layer. A natural extension of the theory of one-step gradient feature learning in feedforward networks suggests that one could first learn the attention weights with a single gradient step, and then train the output layer. We show that this fails for a shallow transformer trained to recognize a star-free language from unbounded-length input. We prove that the transformer can be trained to recognize the language with a three-phase gradient descent: the output layer is updated first, followed by an attention update, and the output layer is fine-tuned in the third phase. In contrast, the language cannot be learned by training only the attention or output layer, or by a single attention-layer update followed by finetuning the output layer. We empirically show that ``attention learning lagging output layer learning" is beneficial in other language recognition, regression, and in-context learning tasks as well.
\end{abstract}

\section{Introduction}

\emph{How do neural architectures learn features and attentions? What is the impact of the learning dynamics on the generalization abilities of the trained model?}
These are major open questions in deep learning theory.
Significant progress has been made to address these questions in feedforward neural networks, both from a broad optimization perspective---for instance, implicit biases of gradient descent \citep{LyuL20}---as well as from the viewpoint of learning specific tasks, like learning parity, sparse polynomial or gaussian mixture classification \citep{DanielyM20,AbbeAM23,FreiCB23}.
A notable theoretical insight has emerged in this context, sometimes referred to as \emph{one-step gradient feature learning} \citep{BaESWWY22,ShiWL23}---for learning tasks using a shallow feedforward network, it often suffices to perform one step of gradient descent on the input layer to learn the features and then train the output layer till convergence.
Decoupling the training of different layers helps to study simplicity bias, lottery tickets, etc., in feedforward networks \citep{AbbeAM23,ShiWL23}.

Much less is understood about how attention learning happens in transformers. There are studies on implicit biases and training dynamics (see Relation to prior work), but it is not clear what the ideal way is to theoretically decouple the training of attention and output layers in a shallow transformer, comprising a softmax attention layer followed by a feedforward output layer. For example, a natural extension of the gradient feature learning paradigm in a shallow transformer is to perform a single gradient descent update of the attention weights, followed by training of the output layer. It is not clear if a \emph{``one-step attention learning''} paradigm is suitable for training shallow transformers or a different multi-phase decoupled training is needed, for example, in context of the prototypical task of learning formal languages \citep{StroblMW0A24}.

\paragraph{Contributions.}
We study how attention learning happens in a shallow transformer trained to recognize formal language, where the training consists of words of varying length (potentially unbounded). Despite active research on  the theory of attention learning, little is known about the training dynamics of transformers when trained on variable-length inputs. We aim to address if a one-step attention learning paradigm---akin to gradient feature learning in feedforward networks \citep{ShiWL23}---is an appropriate simplification of how shallow transformers learn to recognize a formal language. We have the following key contributions.

\textit{Multi-phase gradient descent dynamics in learning the occurrence language.}
We consider a minimalist example of learning to recognize the occurrence language $A_s = \Sigma^* s\Sigma^*$ on the finite alphabet $\Sigma=\{1,\ldots,d\}$, i.e., the language of all words containing $s$, where $s\in\Sigma$ is not known apriori (Figure \ref{fig:overview}a). The occurrence language can be expressed by a shallow transformer with $d$-variate encoding of each bit, a single attention head, and a linear output layer (Proposition \ref{prop:expressivity}). Empirically, for this task, gradient descent (GD) first learns the output layer before the attention weights (Figure \ref{fig:overview}b). Theorem \ref{thm:main_positive} formally shows that it is possible to learn the optimal transformer using multi-phase GD: (Phase 1) one-step GD update of output layer; (Phase 2) one-step GD update of attention weights; (Phase 3) training output layer till convergence. This result indicates that a \emph{one-step attention learning} works in transformers, but when the attention learning happens after output layer is partly trained (see Figure \ref{fig:overview}d-e).

\textit{Impossibility results for simpler learning approaches.}
It is natural to question if the single-head transformer can also be trained using a simpler gradient-based rule.  Propositions \ref{prop:attn_only} and \ref{prop:output_only} show that, if either only the output or the attention layer is trained with the other held fixed, then a shallow transformer cannot learn uniformly over the set of $d$ languages $\{A_s: s\in \Sigma\}$. The same also holds for transformers with $O(\sqrt{d})$ heads (Proposition~\ref{prop:output_only_multi_heads}), but arbitrary fixed attention, showing that lazy training fails for this task, unless with a large number of heads. Furthermore, Theorem~\ref{thm:attention-first-fails} shows that a two-phase approach of one step of GD update of attention, followed by training the output, cannot learn the language. Thus a na{\"i}ve extension of gradient feature learning principle does not work in this case.

\textit{Finite sample analysis and empirical study on other languages.}
While the aforementioned results consider GD on population loss, Corollary \ref{thm:fbGD} extends the learnability result to finite sample setting, showing that the three-phase training also works with $O(d\ln d)$ samples for any fixed tolerance. 
We also numerically study other regression, language recognition, or in-context learning tasks with shallow transformers, and show that in most cases, the three-phase approach (with one-step GD update of output layer before one-step attention learning) outperforms other decoupled training of attention and output layers. Hence, we hypothesize that this approach is potentially a reasonable theoretical simplification of transformer training dynamics.

\begin{figure}[t]
    \centering
    \includegraphics[width=\linewidth]{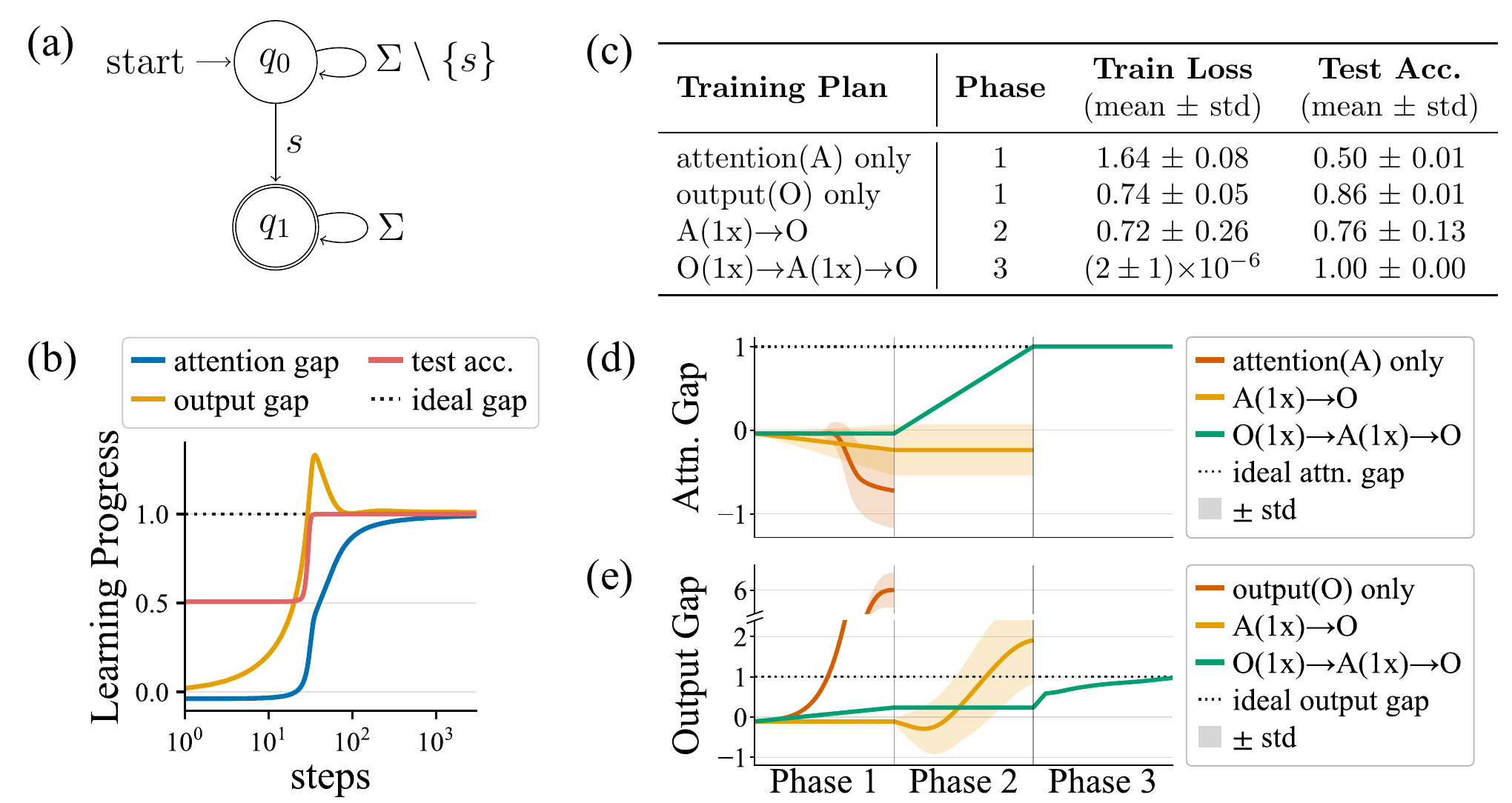}
    \caption{\textbf{(a)} Automaton of the occurrence language $A_s = \Sigma^*s\Sigma^*$ with alphabet $\Sigma = \{1,\ldots,d\}$ and unknown $s\in\Sigma$. The language can be expressed by a transformer with one attention head and linear output layer \eqref{eq:single_head_transformer}, which encodes word $w=(w_1,\ldots,w_m)$ as $(\mathbf{e}_{w_1},\ldots,\mathbf{e}_{w_m})$, where $\mathbf{e}_i \in \mathbb{R}^d$ is $i$-th standard basis vector. The optimal shallow transformer has weights of the output layer $\boldsymbol{\beta} = 2\mathbf{e}_s \in \mathbb{R}^d$, and attention score for symbol $i\in\Sigma$ is $\gamma_i=1$ if $i=s$, and 0 otherwise.
    \textbf{(b)} A shallow transformer is trained to recognize $A_s$ using gradient descent jointly on both layers. The output gap $(\beta_s - \max\limits_{i\neq s} \beta_i)/2$ and attention gap $(\gamma_s - \max\limits_{i\neq s} \gamma_i)$, plotted across training epochs, show that output gap increases before attention gap appears, indicating that attention learning lags output layer learning.
    \textbf{(c)} Several training plans are considered, which decouple gradient descent on attention layer (A) and output layer (O), sometimes using a single gradient step (1x). ``One-step attention learning'', where one attention update is followed by training output layer till loss converges, A(1x)$\to$O, does not correctly learn the language. We later prove that a 3-phase training, O(1x)$\to$A(1x)$\to$O, correctly learns to recognize $A_s$.
    \textbf{(d)--(e)} show how our 3-phase plan reaches the ideal gaps in both layers, but training attention/output only, or attention before output do not.  
    }
    \label{fig:overview}
\end{figure}

\paragraph{Relation to prior work.}
The expressivity and training dynamics of transformers is an active area of research in deep learning theory. While there is a rich literature on the expressivity of transformers in terms of formal languages, circuit complexity or space/time complexity classes (see surveys in  \citet{StroblMW0A24,LinB25,BarceloCCSY25}), the current research on training dynamics of transformers has predominantly focused on finite-dimensional learning problems.
There are results on the lazy training regime of shallow transformers \citep{HronBSN20, ArdaCE26} or implicit biases towards maximum margin \citep{tarzanagh2023}, which hold in general data settings. However, fine-grained characterization of the training dynamics requires studying specific data distributions or learning tasks, which we also assume. 

\citet{WenZLZ25,han2026} study the training dynamics of transformers when trained on sparse parity function---a prototypical task for feedforward networks.
\citet{MarionBBB25,Dohmatob25} study the joint dynamics of attention and output layer of a shallow transformer for single-location regression, while \citet{WangW0L24} derive the joint training dynamics for sparse token averaging. The settings in these works correspond to finite/fixed-dimensional inputs, thereby ignoring the ability of transformers to handle varying input length. Furthermore, the attention-only or joint training dynamics do not reveal the importance of learning different layers. 
\citet{BiettiCBJB23,ChenSWY24} study the dynamics of training induction heads in transformers for next token prediction on bigrams and in-context n-gram Markov chains. \citet{BiettiCBJB23} empirically observe, under joint training, the induction head is learned more slowly than the global bigram model, similar to Figure~\ref{fig:overview}b, and show that the induction head can be learned by sequential steps: first on output matrix, then on two attention key matrices. \citet{ChenSWY24} also consider a multi-phase training of different layers, involving a parameter-inefficient polynomial regression. Both works, however, restrict the problems to fixed context windows, effectively to fixed-length inputs. Moreover, these works do not explore if attention or output needs to be learned first.

Several empirical works study how transformers learn formal languages~\citep{LiuAGKZ23, Liu_glitch, ButoiKSVCD25}, but theoretical characterization of the training dynamics of transformers in learning formal languages is rather recent. \citet{YangKWL24} study the gradient flow dynamics of learning word co-occurrence for fixed-length input, showing that the first phase mainly trains the output layer, before the joint training of output and attention happens. \citet{HuangL025} study the joint training dynamics of even-pairs and parity-check languages (the latter with chain of thought), restricting to bounded-length words, and observe a similar slow learning of attention in early phase. These works implicitly suggest that a phenomenon equivalent to \emph{one-step gradient feature learning} in feedforward networks may not hold for transformers, where attention learning typically happens late. Our work takes a step towards resolving this matter by proving that one-step attention learning mechanism is indeed possible, but needs to follow a step of output layer update, which aligns with the intuition from known joint training dynamics. Such a simplifying characterization of the training dynamics is a key contribution of our work, along with the impossibility of learning via simpler attention-only or attention-first training strategies. Prior works also assume fixed- or bounded-length inputs, which our analysis does not require.
\section{Problem Formulation}
\label{sec:problem}
This section describes the occurrence language recognition task, the training data distribution, the transformer model, and its ability to express the language.

\paragraph{Occurrence language and Training data.}
Let $\Sigma = \{1, 2, \dots, d\}$ be a finite alphabet of size $d\ge 2$. For a symbol \(s\in\Sigma\), the \emph{occurrence language}, defined as $A_s := \Sigma^* s \Sigma^*$, is the set of strings that contain $s$. We encode membership by the label function \(\displaystyle\chi_s(w) := 2 \cdot \mathbf{1}\{s \text{ occurs in } w\} - 1 \in \{\pm 1\}.\) Consider any non-empty word $w=(w_1, \dots, w_m) \in \Sigma^*$ drawn from a distribution, say $\rho$, such that the length is $m \sim \mathrm{Geometric}(p)$, i.e., $\prob(m = k) = p(1-p)^{k-1}$ for $k \geq 1$, and given $m$, the symbols $w_1, \dots, w_m \sim_{i.i.d.} \mathrm{Uniform}(\Sigma)$. Equivalently, after the first symbol, the sequence terminates at each step independently with probability $p$. Setting $p = \frac{1}{d-1}$ exactly balances the proportion of positive/negative samples since, for any \(s\in\Sigma\), it holds that
$\prob_{w\sim\rho}(\chi_s(w)=-1) = \sum_{m=1}^{\infty} \prob\cp{|w| = m} \cdot \cp{1-1/d}^m = \frac12$.
%
We use $D_s$ to denote the joint distribution over $(w, y)$ where $w \sim \rho$ and $y = \chi_s(w) \in \{\pm1\}$. Throughout this paper, we fix a target symbol $s \in \Sigma$, unknown to the learner, and write $D \equiv D_s$, $\chi \equiv \chi_s$ when $s$ is unambiguous from context. The transformer, described next, is trained by gradient descent on the population loss over $D$.

\paragraph{Shallow transformer architecture.}
To recognize the language $A_s$, we consider a transformer with a single attention layer with one head, and a linear output layer. Every $i \in \Sigma$ is embedded as $i$-th standard basis vector $\mathbf{e}_i \in \mathbb{R}^d$, and a word $w = (w_1, \dots, w_m) \in \Sigma^*$ encodes into \((x_1, \dots, x_m)\), where \(x_t = \mathbf{e}_{w_t}\). Positional encoding is not used since membership \(\chi_s(w)\) is permutation-invariant. We prepend a fixed classification token $x_{\mathrm{CLS}}$, which attends to the \(m\) input positions but not to itself. The prediction is computed from the attention output of this token alone. 

Since the model prediction depends only on the output of the classification token, the query and key matrices enter the model only through the vector $\boldsymbol{\alpha} := W_K^\top W_Q\, x_{\mathrm{CLS}} \in \mathbb{R}^d$, whose $i$-th entry is the attention logit that the \texttt{CLS} token assigns to symbol $i$. Likewise, the value matrix and the linear output enter only through their product $\boldsymbol{\beta} \in \mathbb{R}^d$. We treat $(\boldsymbol{\alpha}, \boldsymbol{\beta})$ as the trainable parameters, and call them the attention and output weights, respectively. 
The model computes
\begin{align} \label{eq:single_head_transformer}
    f_{\attn, \readout}(w)
    & = \readout^\top \cp{\sum_{t=1}^m \frac{\exp(\attn^\top x_t)}{\sum_{r=1}^m \exp(\attn^\top x_{r})} \cdot x_t} - 1 \notag \\
    & = \sum_{i=1}^d \beta_i \cdot \frac{C_i(w)\exp(\alpha_i)}{\sum_{j=1}^d C_j(w)\exp(\alpha_j)} - 1,
\end{align}
where the second line groups the attention scores from positions that carry the same symbol by using $C_i(w) := \sum_{j=1}^m \mathbf 1\{w_j = i\}$ to denote the number of occurrences of symbol $i$ in $w$; a fixed bias \(-1\) is used. We use \(\gamma_i(w): = \frac{C_i(w) \exp(\alpha_i)}{\sum_{j=1}^d C_j(w) \exp(\alpha_j)}\) to denote the total attention scores placed on symbol $i\in[d]$. 

\paragraph{Loss and expressivity.}
The transformer is trained to minimize the population (squared) risk over $D$, 
\begin{equation*}
\mathcal{L}(\attn,\readout) = \E_{(w,y)\sim D}\cb{\frac{1}{2}\cp{f_{\attn,\readout}(w) - y}^2}.
\end{equation*}
In Corollary \ref{thm:fbGD}, some results are extended to a finite sample case, where the empirical risk is minimized.

The following expressivity result, proved in Appendix~\ref{app:proof_prop_expressivity}, shows that there is a choice of $\attn,\readout$ for which the loss can be made arbitrarily small.

\begin{proposition}[Shallow transformer expresses occurrence language]\label{prop:expressivity}
   For any specified $\varepsilon\in(0,1)$, a shallow transformer with $\readout=2\mathbf{e}_s$ and $\attn = \ln(4d/\varepsilon^2) \mathbf{e}_s$ achieves  \(\mathcal L(\attn, \readout) \le \varepsilon\).
   The result also holds for any $\attn$ such that \(\alpha_s - \max_{j\ne s} \alpha_j \ge \ln(4d/\varepsilon^2)\).
\end{proposition}

\paragraph{Training.}Our main results on training dynamics assume that transformer is trained to minimise the population risk $\mathcal{L}$ with GD.  A gradient step on the parameters is written as \(\attn' =  \attn - \eta\nabla_{\attn}\mathcal{L}(\attn, \readout), \quad \readout' = \readout - \eta'\nabla_{\readout}\mathcal{L}(\attn, \readout),\)
with parameter-wise step sizes $\eta, \eta' \ge 0$. We subsequently use training plans that decouple GD updates of $\attn,\readout$, that is, each step updates only one of the two parameters, while the other is frozen (i.e., the learning rate for that parameter is \(0\)). Our results in Sections~\ref{sec:impossibility} and~\ref{sec:positive_result} differ in the initialization of $(\attn, \readout)$ and in the order of training the layers, as explained in the respective results. 


\section{Failure of One- and Two-Phase Training}\label{sec:impossibility}
We first show that training a single layer is insufficient to learn the occurrence language, that is, if either layer is fixed independently of the target, then training the other cannot compensate.
\begin{proposition}[Training attention only cannot learn the occurrence language]\label{prop:attn_only}
    There exists an absolute constant $c_1 > 0$, such that for every \(d\ge 3\) and $\tilde{\readout} \in \mathbb{R}^d$,
    \[
        \max_{s \in [d]}\; \inf_{\attn \in \mathbb{R}^d}\;
        \mathcal{L}_s(\attn, \tilde{\readout}) \;\geq\; c_1,
    \]
    that is, for any fixed output, there is a target $s\in\Sigma$ where no attention weight achieves arbitrarily low risk.
\end{proposition}
Since the bound holds for the infimum over $\attn\in\mathbb R^d$, it applies to any procedure for training the attention weights, not only to gradient descent. A similar result holds with the roles of the two layers exchanged.
\begin{proposition}[Training output only cannot learn the occurrence language]\label{prop:output_only}
    There exists an absolute constant $c_2 > 0$, such that for every \(d\ge 3\) and $\tilde{\attn} \in \mathbb{R}^d$,
    \[
        \max_{s \in [d]}\; \inf_{\readout \in \mathbb{R}^d}\;
        \mathcal{L}_s(\tilde{\attn}, {\readout}) \;\geq\; c_2,
    \]
    that is, for any fixed attention, there is a symbol \(s\), for which no output can achieve an arbitrarily small risk.
\end{proposition}

Proposition~\ref{prop:output_only} concerns training the output over only one attention head. A natural objection is that fixed attention may suffice with enough heads. In response to this, consider the multi-head extension of \eqref{eq:single_head_transformer}:
\begin{equation}\label{eq:multi_head_output}
    f_{\mathcal{A}, \mathcal{B}}(w) = \sum_{\substack{h\in[H]\\i\in[d]}} \beta_i^{(h)} \frac{C_i(w) \exp\cp{\alpha^{(h)}_i}}{\sum_{j=1}^d C_j(w) \exp\cp{\alpha^{(h)}_j}} - 1,
\end{equation}
where $H$ is the number of heads. One can write the attention weights across heads as \(\mathcal A = (\attn^{(1)}, \dots, \attn^{(H)}) \in \mathbb R^{d\times H}\) and the output weights \(\mathcal B = (\readout^{(1)}, \dots, \readout^{(H)}) \in \mathbb R^{d\times H}\). With $H = d$ heads, fixed attention does suffice: if each head $h\in[H]$ are initialized to attend to symbol $h$, i.e., $\attn^{(h)}$ satisfies the configuration in Proposition~\ref{prop:expressivity}, then for any target $s$, when the output weights converge to $\readout^{(s)} = 2 \mathbf e_s$ and $\readout^{(h)} = \mathbf{0}_d$ for $h \neq s$, the transformer achieve arbitrarily small risk. Further, in the limit $H \to \infty$ with random attention weights, training the output layer becomes kernel regression with a deterministic kernel~\citep{HronBSN20}. The next result shows that substantially fewer than $d$ heads, specifically $O(\sqrt{d})$ are not enough.
 
\begin{proposition}[Risk lower bound for output-only training with multiple heads]\label{prop:output_only_multi_heads}
    There is an absolute constant \(c\in(0,1)\), such that when the number of heads \(H\le c\sqrt d\), for any \(d\ge 3\) and fixed attention \(\tilde{\mathcal A}\),
    \[
        \max_{s\in[d]} \min_{\|\mathcal B\|_F \le d^{1/4}} \mathcal L_s(\tilde{\mathcal A}, \mathcal B) \ge c.
    \]
\end{proposition}
The norm constraint $\|\mathcal{B}\|_F \leq d^{1/4}$ is a requirement of our proof technique, and we do not know whether it can be removed. It is, however, not restrictive, since every solution we are aware of has a norm that is constant or decays with $H$. Consider the training plan in Theorem~\ref{thm:main_positive} for two cases: if all heads but one are frozen with zero output weights, the single-head solution can be achieved, whose output weights have constant norm \(\|\mathcal B\|_F = 2\). If all $H$ heads are trained from identical zero initialization, they remain identical by symmetry, and the parameter \(\attn^{(h)}\) will converge to satisfy the configuration in Proposition~\ref{prop:expressivity}, while \(\readout^{(h)}\) will converge to \(\frac{2}{H}\mathbf e_s\) for each head \(h\in[H]\). In this case, the output norm is $\|\mathcal{B}\|_F = 2/\sqrt{H}$.

\paragraph{Two-phase training with one-step attention learning.}
Propositions~\ref{prop:attn_only}--\ref{prop:output_only_multi_heads} separate the occurrence language from settings in which training a single layer suffices, whether learning the output layer over a fixed representation~\citep{HronBSN20,ArdaCE26} or learning the features under a fixed output layer~\citep{FreiCB23,han2026,tarzanagh2023}---both approaches lead to the risk bounded away from zero for some target. We next consider a two-phase plan: a single gradient step on the attention weights, followed by training the output layer. This scheme is an adaptation of the one-step gradient feature learning in feed-forward networks, where it provably learns certain target classes~\citep{BaESWWY22, ShiWL23}. As we show next, this 2-phase GD training also cannot learn the occurrence language if one starts from uniform attention $\attn=\mathbf{0}_d$.


\begin{theorem}[Attention-first two-phase GD fails to learn occurrence]
\label{thm:attention-first-fails}
There is absolute constant $c_3 > 0$ such that the following holds for every $d \geq 9$, $\tilde{\readout} \in \mathbb{R}^d$, and $\eta > 0$. 
Suppose the transformer is initialized as $\attn = \mathbf{0}_d$ (uniform attention) and ${\readout}=\tilde{\readout}$, and one gradient step is used to update the attention weights as
    \[
        \attn'_s := -\eta \nabla_{\attn}\mathcal{L}_s(\mathbf{0}_d, \tilde{\readout}),
    \]
where $s\in\Sigma$ is unknown but influences the updated attention $\attn'_s$ through the population risk gradient. Then
    \[
        \max_{s \in [d]}\; \inf_{\readout \in \mathbb{R}^d}\;
        \mathcal{L}_s(\attn'_s, \readout) \;\geq\; c_3,
    \]
    that is, there is a target $s\in\Sigma$ for which the optimal output layer on top of the updated attention $\attn'_s$ cannot achieve an arbitrarily small risk.
\end{theorem}

The proofs of all results in this section are deferred to
Appendix~\ref{app:proof_prop_attn_only}-\ref{app:proof_attention_first_fails}. Together, these impossibility results show that neither layer can be trained alone to learn the occurrence language, and that taking the attention step first does not help. In Section~\ref{sec:positive_result}, we show that a three-phase training plan (training the output, then the attention, then the output again) does help the transformer to learn the language.
\section{Learning with Three-Phase Training}\label{sec:positive_result}
The GD dynamics in Figure~\ref{fig:overview}b show that output weights separate before the attention weights move. This is in line observations of \citet{BiettiCBJB23,HuangL025} about joint gradient dynamics in other settings. Motivated by this, we study the following three-phase training plan. We initialize $\attn = \readout = \mathbf{0}_d$, which is agnostic to $s$, and train as follows:

\textit{Phase 1:} One GD step on the output with step size $\eta_1$ and frozen attention: $\readout' = -\eta_1 \nabla_{\readout}\mathcal{L}_s(\mathbf{0}_d, \mathbf{0}_d)$.

\textit{Phase 2:} One GD step on the attention weights with step size $\eta_2$ and frozen output: $\attn' = -\eta_2 \nabla_{\attn}\mathcal{L}_s(\mathbf{0}_d, \readout')$.

\textit{Phase 3:} $T$ steps of GD on the output weights with step size $\eta_3$, starting from $\readout'$ with the attention fixed at $\attn'$. Formally, $\readout(0) = \readout'$ and $\readout(t+1) = \readout(t) - \eta_3 \nabla_{\readout} \mathcal{L}_s(\attn', \readout(t))$ for $t = 0, \dots, T-1$. We denote the resulting output weights by $\readout'' := \readout(T)$.

The following theorem shows that this training plan learns the occurrence language for every target symbol up to any specified arbitrarily small risk bound.

\begin{theorem}[Three-phase training learns the occurrence language]\label{thm:main_positive}
    For every \(d\ge 3\), every target symbol \(s\in[d]\), and every \(\varepsilon\in(0,1)\), when the step sizes and steps satisfy: \(0 < \eta_1<\frac{d-1}{2},\; \eta_2 > \frac{20(d-1)^2}{\eta_1}\ln\frac{8d}{\varepsilon}, \; \eta_3 = 1,\;T \geq 4d\ln\frac{16d}{\varepsilon}\), then the parameters $(\attn', \readout '')$ obtained by the above 3-phase training plan satisfy:
    \[
        \mathcal{L}_s(\attn', \readout'') \le \varepsilon.
    \]
\end{theorem}
We present the complete proof in Appendix~\ref{app:proof_thm_4}, and discuss three implications of the result below. 

\begin{figure}[t]
    \centering
    \includegraphics[width=\linewidth]{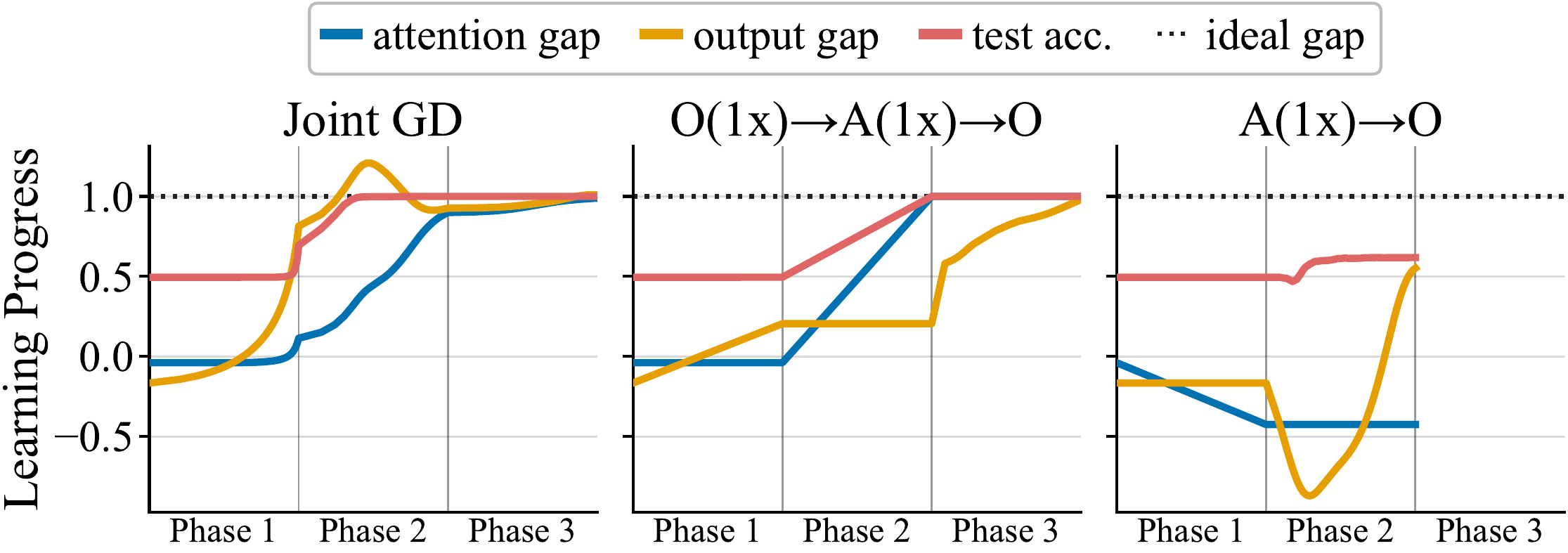}
    \caption{\textbf{Joint GD (left), O-A-O training (middle), and A-O training (right).} Attention gap, output gap, and test accuracy for $d = 50$ are shown with \(5d^2\) training and \(d^2\) test samples. Joint GD has no phases; for comparison, we start its Phase 2 when the attention starts to move (gap = \(0.1\), at step \(31\)) and end its Phase 2 when the attention almost saturates (gap = \(0.9\), at step \(127\)). In the 3-phase plan, the test accuracy reaches $1$ by the end of Phase 2. Phase 3 only reduces training loss by fine-tuning output layer.}
    \label{fig:gd_comparison}
\end{figure}

\textbf{The role of each phase.}
Figure~\ref{fig:gd_comparison} compares joint GD, in which both parameters are updated at every step, with the three-phase plan and the attention-first plan. Phase 1 cannot be omitted, otherwise the plan is the attention-first schedule of Theorem~\ref{thm:attention-first-fails}, in which the attention step receives no guidance from the output layer and may not move towards the target (Figure~\ref{fig:gd_comparison}c). Phase 1 also mirrors the early stage of joint GD, where the output separates while the attention weights barely move (Figure~\ref{fig:gd_comparison}a). 
Phase 2 then saturates the attention on the target with a single large step, and Phase 3 trains the output weights towards $2\mathbf{e}_s$.

The three-phase training plan differs from joint GD in two ways. First, in joint GD, the output weight on the target symbol can overshoot and then decrease while the attention saturates; freezing the output in Phase 2 removes this effect. Second, in joint GD, attention saturation and output weight convergence occur simultaneously. Our analysis, based on decoupling the two layers, shows that these aspects of joint GD are not necessary for learning. However, our analysis cannot explain this specific behavior of joint GD, which requires analysis of joint training dynamics.

\textbf{Initialization} Although zero initialization ($\attn = \readout = \mathbf{0}_d$)  is agnostic to $s$, it is not standard in practice. We adopt this to simplify the theoretical analysis, as also done in~\citet{tarzanagh2023}. Empirically, the same behavior occurs under random initialization (Table~\ref{tab:complete_experiments}). 


\textbf{Finite-Sample GD analysis.} The population gradient is not accessible in practice. We extend Theorem \ref{thm:main_positive} to finite-sample GD to show that $O(d\ln d + (\ln d)/\varepsilon)$ samples suffice to reach a population risk of at most \(\varepsilon\).


\begin{corollary}[Three-phase training with finite samples]\label{thm:fbGD}
    Consider the same initialization \(\attn = \readout = \mathbf 0_d\), and the three-phase plan where the gradients of Phases 1 and 2 are computed on the empirical risk of $N_1$ i.i.d.\ samples from $D$, and the gradients of Phase~3 on the empirical risk of $N_3$ fresh i.i.d.\ samples. For every $c > 0$, there exists a constant $C > 0$ such that the following holds for every $d \geq 3$, every target symbol $s \in [d]$, and every $\varepsilon \in (0, 1)$. If \(0<\eta_1 \leq \sqrt{d}, \; \eta_2 \geq \frac{C d^2}{\eta_1}\ln\frac{d}{\varepsilon}, \; \eta_3 = 1, \; T = \left\lceil \frac{72}{\varepsilon} \right\rceil\), and \(N_1 \geq C d \ln d, \; N_3 \geq \frac{C \ln d}{\varepsilon}\),
    then with probability at least $1 - 3d^{-c}$ over the two draws, the resulting parameters $(\hat{\attn}', \hat{\readout}'')$ achieve
   $\mathcal{L}_s\cp{\hat{\attn}', \hat{\readout}''} \le \varepsilon$.
\end{corollary}
The proof is deferred to Appendix~\ref{app:proof_cor_1}. Note that Phase 3 uses a batch of samples independent of those used in Phases 1 and 2. Redrawing samples between training stages is a standard approach in layer-wise training: \citet{BaESWWY22} finetune the second layer with a new set of samples after one GD step on the first layer, \citet{NichaniDL23} also adopt similar sample splitting technique in their layer-wise training plan.

In contrast to the one- or two-phase training plans, the three-phase plan learns the occurrence language for every target symbol with polynomially many samples.



\section{Experiments}\label{sec:experiments}
We empirically examine if our three-phase plan learns other language, regression, and in-context learning tasks.
We present the setup and discuss the results.

\textbf{Tasks.} We consider six tasks: (i) occurrence; (ii) \texttt{OR}, where a string is positive if it contains at least one of two unknown target symbols; (iii) \texttt{AND}, where a string is positive if it contains both of two unknown target symbols; (iv) \texttt{even pairs}, the language over the binary alphabet $\{a, b\}$ of strings where $ab$ and $ba$ occur even number of times~\citep{HuangL025}; (v) in-context bigram learning (\texttt{ICL bigram}), a simplified version of \citet{BiettiCBJB23} where the first attention layer is fixed, so that each position locks onto its preceding token; and (vi) single-location regression (\texttt{SLR})~\citep{MarionBBB25}. For occurrence and \texttt{OR}, we use the single-head model in~\eqref{eq:single_head_transformer}, and for \texttt{AND} its multi-head extension~\eqref{eq:multi_head_output} with trainable bias and four heads, since one head cannot express it. For others, we use one-layer, single-head models with similar parameters to our attention and output (see Appendix~\ref{app:experiment_setup}).

\textbf{Training plans.}
We compare joint GD, as a baseline, with five other training plans: (i) attention only (A), corresponds to Proposition~\ref{prop:attn_only}; (ii) output only (O), corresponds to Proposition~\ref{prop:output_only}; (iii) A(1x)$\to$O, matches the plan in Theorem~\ref{thm:attention-first-fails}; (iv): O(1x)$\to$A(1x)$\to$O, following the phases in Theorem~\ref{thm:main_positive}; and (v) O(1x)$\to$A(1x): one output step followed by one attention step with a larger step size. The last plan isolates the role of the final output phase. All plans are run from a random initialization, with O(1x)$\to$A(1x)$\to$O additionally from the fixed initialization in Section~\ref{sec:positive_result}.

\begin{table*}[t]
\centering
\caption{\textbf{Train loss and test accuracy (loss for \texttt{SLR}), with mean over $50$ seeds}. Standard deviations are reported in Appendix~\ref{app:table_with_std}. Values below \(10^{-3}\) are reported with their upper bounds; those that match or surpass the joint GD baseline are bold. For \texttt{AND}, all attention-training phases have multiple steps instead of one.}\label{tab:complete_experiments}
\label{tab:complete_experiments}
\setlength{\tabcolsep}{4pt}
\resizebox{\textwidth}{!}{%
\begin{tabular}{@{}llcccccccccccc@{}}
\toprule
\multicolumn{2}{@{}c}{\multirow{2}{*}[-1.2ex]{\textbf{Training plan}}} & \multicolumn{2}{c}{\textbf{Occurrence}} & \multicolumn{2}{c}{\textbf{OR}} & \multicolumn{2}{c}{\textbf{AND}} & \multicolumn{2}{c}{\textbf{Even-Pair}} & \multicolumn{2}{c}{\textbf{ICL bigram}} & \multicolumn{2}{c}{\textbf{SLR}} \\
\cmidrule(lr){3-4}\cmidrule(lr){5-6}\cmidrule(lr){7-8}\cmidrule(lr){9-10}\cmidrule(lr){11-12}\cmidrule(lr){13-14}
 & & \hd{Train}{Loss} & \hd{Test}{Acc.} & \hd{Train}{Loss} & \hd{Test}{Acc.} & \hd{Train}{Loss} & \hd{Test}{Acc.} & \hd{Train}{Loss} & \hd{Test}{Acc.} & \hd{Train}{Loss} & \hd{Test}{Acc.} & \hd{Train}{Loss} & \hd{Test}{Loss} \\
\midrule
Coupled with rand init.& Joint GD & \(10^{-3}\) & \(1.00\) & \(10^{-3}\) & \(1.00\) & \(0.04\) & \(0.99\) & \(0.01\) & \(1.00\) & \(10^{-3}\) & \(1.00\) & \(0.01\) & \(0.01\) \\
\midrule
\multirow{5}{*}{\shortstack[l]{Decoupled 1/2/3-phase\\with random init.}}
 & Attention (A) only & \(1.64\) & \(0.50\) & \(1.67\) & \(0.50\) & \(0.24\) & 0.94 & \(0.69\) & \(0.51\) & \(3.90\) & \(0.02\) & \(0.50\) & \(0.50\) \\
 & Output (O) only & \(0.74\) & \(0.86\) & \(0.68\) & \(0.85\) & \(0.79\) & \(0.75\) & \(0.63\) & \(0.50\) & \(2.70\) & \(0.16\) & \(0.48\) & \(0.48\) \\
 & A(1x)$\to$O & \(0.72\) & \(0.76\) & \(0.69\) & \(0.75\) & \(0.30\) & \(0.91\) & \(0.62\) & \(0.50\) & \(2.58\) & \(0.26\) & \(0.41\) & \(0.43\) \\
 & O(1x)$\to$A(1x) & \(0.47\) & \good{1.00} & \(0.64\) & \(0.78\) & \(0.52\) & \(0.87\) & \(0.52\) & \(0.50\) & \(3.20\) & \good{1.00} & \(0.28\) & \(0.28\) \\
 & O(1x)$\to$A(1x)$\to$O & \good{10^{-5}} & \good{1.00} & \good{10^{-5}} & \good{1.00} & \(0.27\) & 0.93 & \good{0.08} & \good{1.00} & \good{10^{-3}} & \good{1.00} & \good{0.01} & \good{0.01} \\
\midrule
Sec \ref{sec:positive_result}: 3-phase, zero init.& O(1x)$\to$A(1x)$\to$O & \good{10^{-8}} & \good{1.00} & \good{10^{-7}} & \good{1.00} & 0.90 & 0.68 & \good{0.08} & \good{1.00} & \good{10^{-3}} & \good{1.00} & \good{0.01} & \good{0.01} \\
\bottomrule
\end{tabular}%
}
\end{table*}

We use $d = 50$, $N = 5d^2$ training samples, and $d^2$ samples for testing, with the following two exceptions: for \texttt{even pairs}, we train and test on $N = 2046$ strings of length at most $10$~\citep{HuangL025}; for \texttt{SLR}, we fix input length $L = 10$~\citep{MarionBBB25}. All tasks are trained with Mean Squared Error (MSE), except \texttt{even pairs} and \texttt{ICL bigram}, which use the logistic and cross-entropy loss. Other setups can be found in Appendix~\ref{app:experiment_setup}. We report the train loss and the test accuracy (test loss for \texttt{SLR}) as mean over $50$ seeds.

\paragraph{Results.}
Table~\ref{tab:complete_experiments} shows the performance of various training plans and Figure~\ref{fig:trajectory_complete} visualizes corresponding learning progress. We make the following observations.

\paragraph{Which training plans learn each task?}
Joint GD learns all six tasks, with accuracy at least $0.99$ on the five classification tasks and test loss $0.01$ on \texttt{SLR}. The three-phase plan matches it on five of them, from both the random and the zero initialization: it reaches accuracy $1.00$ on occurrence, \texttt{OR}, \texttt{even pairs}, and \texttt{ICL bigram}, and test loss $0.01$ on \texttt{SLR}. No one- or two-phase plan achieves this across the five tasks.

The O(1x)$\to$A(1x) plan shows the contribution of the final phase. On occurrence and \texttt{ICL bigram}, it reaches perfect accuracy, but the training loss remains high. On these tasks, two phases suffice to classify, and the third is needed to reduce the loss. On the other tasks, two phases are not enough, even for classification.

\texttt{AND} is the exception. From the random initialization, the three-phase plan only reaches an accuracy of $0.93$, below joint GD ($0.99$), and attention-only training matches it ($0.94$). The three-phase plan with zero initialization performs worse than with random initialization, since all heads remain identical throughout training and cannot express the language. We return to further discussions of \texttt{AND} in Section~\ref{sec:discussion}.

\paragraph{Learning progress of joint GD and the phased plans.}
The final train loss and test accuracy / loss show that the three-phase plan can learn other formal languages, but they do not show how the output and attention weights evolve during the learning process. We therefore ask two further questions. \emph{First, does joint GD exhibit the same order of learning on these tasks as on the occurrence language, with the output weights separating before the attention weights move? Second, how closely do the trajectories of different phased plans follow that of joint GD?} 

Figure~\ref{fig:trajectory_complete} addresses both questions. First, on \texttt{AND}, \texttt{ICL bigram}, and \texttt{SLR}, joint GD follows the same order as on occurrence and \texttt{OR}: the attention gap stays near zero while the output gap grows, and only begins to rise when the output separates significantly. \texttt{even pairs} is slightly different: under joint GD, the output only leads at a very short period, and the attention grows as fast as the output during the remaining training time, and the attention saturates before the output does. However, this still satisfies the principle that attention lags output learning at the start of learning.

For the second question, the three-phase plan yields the same final gaps as joint GD on \texttt{even pairs}, \texttt{ICL bigram}, and \texttt{SLR}. On \texttt{ICL bigram} and \texttt{SLR}, the trajectories follow a similar path as joint GD. On \texttt{even pairs}, the paths differ: the first output step of the three-phase plan already brings the gap close to its ideal value, whereas under joint GD the output gap grows slowly throughout. On \texttt{AND}, joint GD overshoots the ideal output gap while the attention saturates, as on the occurrence language, whereas the output gap of the three-phase plan remains below the ideal value, consistent with its lower accuracy in Table~\ref{tab:complete_experiments}.

The other plans show how the order matters. In A(1x)$\to$O, the attention step leaves the attention gap near zero on \texttt{ICL bigram} and \texttt{SLR}, and the subsequent output training cannot compensate. On \texttt{AND}, in contrast, the attention gap grows even when the attention is trained first or alone, possibly because the anti-symmetric initialization ensures that some heads favor a target symbol from the start. Finally, on \texttt{even pairs}, \texttt{ICL bigram} and \texttt{SLR}, output-only training aligns the output with the ideal direction, yet Table~\ref{tab:complete_experiments} shows that it does not learn the task: correct output alone is insufficient without correct attention.

\begin{figure}[htb]
    \centering
    \includegraphics[width=\linewidth]{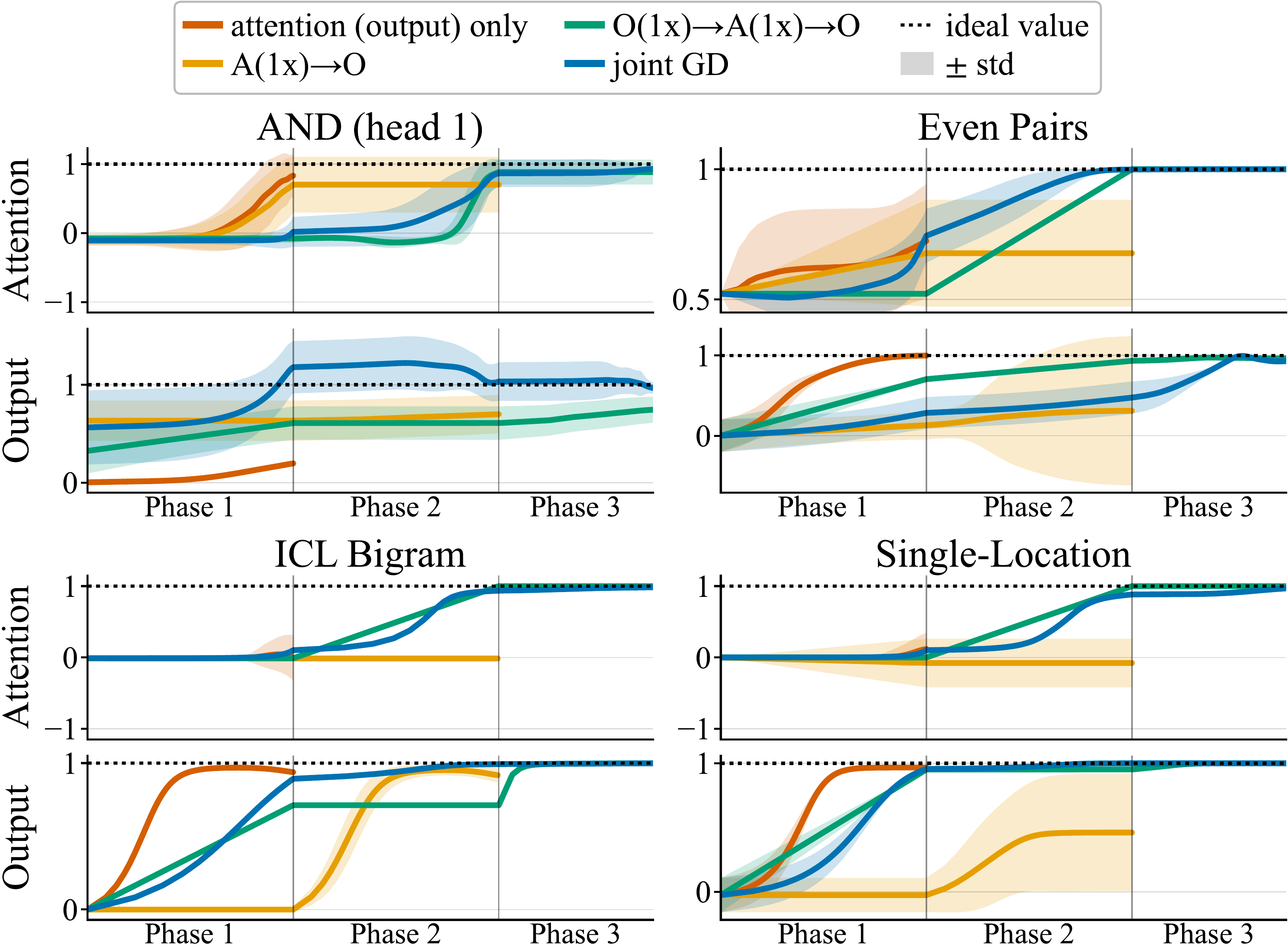}
    \caption{\textbf{Attention and output gaps on four tasks.} Mean $\pm$ standard deviation over $50$ runs; dotted lines mark the ideal gaps (see Appendix~\ref{app:gaps}). Phases are defined as in Figure~\ref{fig:gd_comparison}; single-layer plans end after Phase 1 and A$\to$O after Phase 2. For \texttt{AND}, we plot the head with the largest gap for each plan (all heads: Appendix~\ref{app:and_all_heads}). 
    For \texttt{even pairs} and \texttt{SLR}, the output gap is $\cos(u, w^*)$ and $\cos(v, v^*)$, with $w^*$ the max-margin and $v^*$ the ground-truth direction. Since $w^*$ depends on the attention weights, the output gap on \texttt{even pairs} can change in Phase 2 although the output is frozen. 
    O$\to$A and the zero initialization are omitted, as they coincide with the first two phases of O$\to$A$\to$O and with the random initialization, respectively. The trajectories on \texttt{OR} resemble those on occurrence in Figure~\ref{fig:gd_comparison} and are deferred to Appendix~\ref{app:learning_progress_occur_or}.}
    \label{fig:trajectory_complete}
\end{figure}

\section{Discussions}\label{sec:discussion}

In this work, we study the (decoupled) training dynamics of a shallow transformer that recognizes a star-free language---the occurrence language---from inputs of unbounded length. We prove that one-step attention learning succeeds when it lags output layer learning. The natural two-phase adaptation of gradient feature learning---a single attention update followed by training the output layer---provably fails for some target symbol, as does training either layer alone (Propositions \ref{prop:attn_only}--\ref{prop:output_only_multi_heads}, Theorem \ref{thm:attention-first-fails}). However, a three-phase training plan---involving \emph{one-step attention learning}, sandwiched between two phases of output layer training---suffices to learn the occurrence language (Theorem \ref{thm:main_positive}, Corollary \ref{thm:fbGD}). The underlying mechanism is simple: the gradient of the attention is driven by the separation of the output, so the attention receives no signal about the target symbol until the output layer is informative.

We expect this mechanism to be relevant beyond the specific problem studied in this paper. Our experiments already show that, in many cases, attention learning lags output learning, and the accuracy rises only once the attention begins to saturate. This may help explain the ``abrupt learning'' observed in transformers~\citep{SinghMHCS24, GopalaniH25}. Our impossibility results also suggest that parameter-efficient methods where only the attention or the output layer is finetuned can be fundamentally limited, and that the common theoretical assumption of a faster-learning output layer~\citep{TianWCD23, ChenSWY24} may reflect a necessary order of learning. 
A broad open problem is to establish these connections for deeper models and richer tasks. Our multi-phased training paradigm could complement existing works on joint training dynamics of shallow Transformers \citep{YangKWL24,HuangL025} to investigate this question. We also note specific questions to build upon the presented results.

\textbf{Finite sample complexity and realistic transformer.} The focus of the present study was on the interplay of (decoupled) attention ans output layer training. To derive this, we made two key simplifications---focused on population training risk and parameterized the transformer by $(\attn,\readout)$. Parameterization in terms of key, query, value matrices leads to an over-parametrized system in this case and, as shown in \citet{tarzanagh2023}, leads to different dynamics and implicit biases of GD.
While the finite sample analysis in Corollary \ref{thm:fbGD} shows that $O(d \ln d)$ samples suffice, we do not know if such a sample complexity is necessary for the occurrence language. Furthermore, we assume a batch-GD-style split of training data across phases, and the complexity could differ considerably depending on use of full batch GD or stochastic GD (SGD)---a topic that is currently investigated for parity learning \citep{AbbeAM23}. 


\textbf{Attention-first training with multiple heads.} Table~\ref{tab:complete_experiments} shows that the \texttt{AND} task is the one case where the three-phase plan does not match joint GD, and where training the attention alone does equally well. A possible explanation is from the multi-head model we use for this task. With four heads and a random output layer, the output weights for some heads may already separate a target symbol at initialization, so that their attention receives a useful gradient without a preceding output step. This mirrors prior work on feedforward networks~\citep{BaESWWY22, ShiWL23}, where wider networks can make feature-first learning sufficient. It suggests that the lower bounds for single-layer and attention-first training weaken as the number of heads increases, and that the training order matters most for narrow and shallow models. How many heads are needed for attention-only or attention-first training to succeed is an open question.

\textbf{Learnability and the expressivity hierarchy for transformers.}
The expressive power of transformers over formal languages is well studied: depending on the attention model, depth, and position embeddings, they recognize the star-free languages and
beyond~\citep{StroblMW0A24, BarceloKLP24}. However, expressivity does not guarantee learnability, and our work is only a first step towards a theory of transformer learnability. We establish lower bounds and a convergence guarantee for a single language, and show empirically that the three-phase plan extends to a few other star-free languages. These languages also form a very limited subset of the star-free class: occurrence, \texttt{OR}, and \texttt{AND} are piecewise testable, among the simplest star-free languages, and \texttt{even pairs} compares only the first and last symbols. Characterizing which formal languages are learnable or not learnable by transformers with GD remains wide open. A multi-phase training leads to simple decoupled dynamics and could be potentially useful to answer this question.



\subsection*{AI use statement}

In this work, we used generative AI tools to assist in the writing of proofs and provide feedback on research methodology or experiments.
We have not used generative AI tools to generate synthetic datasets, help develop theoretical models or conceptual frameworks, formulate mathematical claims, provide critical ingredients for proving mathematical claims, propose or refine hypotheses, implement methods, support data analysis, or interpret results. Cleaning and reformating datasets, and assisting with translation, are not applicable to this work.

Additionally, we used generative AI tools to improve scientific figures, suggest experimental parameters, edit software code, and identify relevant literature.

We have reviewed all AI-assisted work. LLM-generated code was verified and refactored by human authors; we use LLM to restructure the proof and check line by line afterward.
We take responsibility for the final content of this work, including text, claims, or artifacts produced with the aid of generative AI.

\subsection*{Acknowledgment}
This work has been supported by the German Research Foundation (DFG) through the Research Training Group GRK 2428 ConVeY.

\bibliography{reference}

\clearpage
\appendix
\thispagestyle{empty}

\onecolumn

\section{Proofs for the main results}\label{app:all_proofs}
\subsection{Notations and Preliminaries}
We introduce some notations, and provide a few calculations that is useful for the remainder of the proofs.
The general per-sample gradients of \(\ell(w)\) with respect to each parameter for each index \(i\in[d]\) is listed as follows:
\begin{align*}
    \fracpartial{\ell(w)}{\beta_i} & = (\hat y (w) - y(w))\cdot \gamma_i\\
    \fracpartial{\ell(w)}{\alpha_i} 
    & = (\hat y (w) - y(w))\cdot \cp{\beta_i\gamma_i(1-\gamma_i) - \sum_{j\ne i}\beta_j\gamma_i\gamma_j}
\end{align*}
We write the expectations of first and the second conditional moments of \(\gamma_s\) at initialisation \(\attn = \mathbf 0_d\):
\begin{align}
    \E[\gamma_s \given y =1] 
    & = \sum_{m=1}^{\infty} \prob\cp{|w| = m \given y=1}\cdot \sum_{n=1}^{m} \prob\cp{C_s(w) = n \given |w|=m, y=1} \cdot \frac{n}{m} \notag\\
    & = \sum_{m=1}^{\infty} \frac{\prob(|w|=m)\cdot \prob(y=1  \given |w| =m)}{\prob(y=1)} \cdot \sum_{n=1}^{m}\frac{\prob(|w|=m)\cdot \prob(C_s(w) = n  \given |w| = m)}{\prob(|w|=m)\cdot \prob(y=1  \given |w|=m)} \cdot \frac{n}{m} \notag\\
    & = 2 \sum_{m=1}^{\infty} \prob(|w| = m) \cdot \E_{n\sim \text{bin}(m, \frac{1}{d})}\cb{\frac{n}{m}} = \frac 2 d \label{eq:positive_mu1}\\
    \E\cb{\gamma_s^2 \given y =1} 
    & = 2\sum_{m=1}^\infty (1-p)^{m-1}p\cdot \E_{n\sim\text{bin}(m, \frac{1}{d})}\cb{\cp{\frac{n}{m}}^2} \notag\\
    & = 2\sum_{m=1}^\infty (1-p)^{m-1}p\cdot\frac{m-1+d}{md^2} = 2 \sum_{m=1}^\infty (1-p)^{m-1}p\cdot\cp{\frac{1}{d^2} -\frac{1}{md^2} + \frac{1}{md}} \notag \\
    & = 2 \cb{\frac{1}{d^2} + \cp{\frac{1}{d} - \frac{1}{d^2}}\cdot \sum_{m=1}^\infty (1-p)^{m-1}p \cdot \frac{1}{m}} \notag\\
    & = 2 \cb{\frac{1}{d^2}+\frac{(d-1)\cdot p}{d^2\cdot (1-p)}\cdot -\ln\cp{p}} = 2 \cb{\frac{1}{d^2}+\frac{(d-1)\cdot\ln(d-1)}{d^2(d-2)}} \label{eq:positive_mu2}\\
    %
\end{align}
For the background positions \(j\ne s\), we have:
\begin{equation}\label{eq:gammaj_positive_mu1}
        \E[\gamma_j \given y = 1]  = \frac{1}{d-1}\cp{1-\frac{2}{d}}, \; \E[\gamma_j \given y = -1] = \frac{1}{d-1}\\
\end{equation}

\begin{align}
    \E[\gamma_j^2 \given y = 1]
    & = \sum_{m=1}^{\infty} \prob\cp{|w|=m \given y=1} \sum_{n=1}^{m}\prob\cp{C_s(w) = n \given y=1, |w| = m}\cdot \E_{l\sim \text{bin}\cp{m-n, \frac{1}{d-1}}}\cb{\cp{\frac{l}{m}}^2} \notag\\
    & =  \sum_{m=1}^{\infty} \prob\cp{|w|=m \given y=1} \sum_{n=1}^{m}\prob\cp{C_s(w) = n \given y=1, |w| = m} \notag\\
    & \qquad \qquad \qquad \qquad \qquad \qquad \qquad \qquad \qquad \cdot \cb{\frac{(m-n)(m-n-1)\cdot \frac{1}{(d-1)^2}+(m-n)\cdot\frac{1}{d-1}}{m^2}} \notag\\
    & \leq \sum_{m=1}^{\infty} \prob\cp{|w|=m \given y=1} \cdot \cb{\frac{1}{(d-1)^2}+\frac{1}{m(d-1)}} \notag \\
    & = \sum_{m=1}^{\infty} \frac{(1-p)^{m-1}p \cdot \cp{1-\cp{\frac{d-1}{d}}^{m}}}{\frac{1}{2}}\cdot \cb{\frac{1}{(d-1)^2}+\frac{1}{m(d-1)}} \notag\\
    & \leq 2 \sum_{m=1}^{\infty} (1-p)^{m-1}p\cdot \cb{\frac{1}{(d-1)^2}+\frac{1}{m(d-1)}} \notag\\
    & = \frac{2}{(d-1)^2} \sum_{m=1}^{\infty} (1-p)^{m-1} p  + \frac{2p}{(1-p)(d-1)}\cdot \sum_{m=1}^{\infty} (1-p)^m\cdot \frac{1}{m} \notag \\
    & = \frac{2}{(d-1)^2} + \frac{2p}{(1-p)(d-1)}\cdot -\ln\cp{1-(1-p)} = \frac{2}{(d-1)^2} + \frac{2\ln(d-1)}{(d-1)(d-2)} \label{eq:gammaj_positive_mu2}\\
    \E\cb{\gamma_j^2 \given y=-1} 
    & = \sum_{m=1}^{\infty} \prob(|w|=m \given y=-1) \sum_{n=0}^{m} \prob\cp{C_j(w) = n \given |w| = m, y=-1}\cdot \cp{\frac{n}{m}}^2 \notag\\
    & = \sum_{m=1}^{\infty} \prob(|w|=m \given y=-1) \cdot \E_{n\sim \text{bin}\cp{m, \frac{1}{d-1}}}\cb{\cp{\frac{n}{m}}^2} \notag\\
    & = \sum_{m=1}^\infty \frac{(1-p)^{m-1}\cdot p \cdot \cp{\frac{d-1}{d}}^{m}}{\frac{1}{2}} \cdot \cb{\frac{m(m-1) \cdot \frac{1}{(d-1)^2}+m\cdot \frac{1}{d-1}}{m^2}} \notag\\
    & \le \sum_{m=1}^\infty 2\cdot (1-p)^{m-1}\cdot p \cdot \cp{\frac{d-1}{d}}^{m} \cdot \cb{\frac{1}{(d-1)^2}+\frac{1}{d-1}\cdot \frac{1}{m}} \notag\\
    & = \frac{2p\cdot (d-1)}{d(d-1)^2} \sum_{m=0}^{\infty} \cb{(1-p)\cdot\cp{\frac{d-1}{d}}}^{m} + \frac{2p}{(1-p)(d-1)} \sum_{m=1}^{\infty} \cb{(1-p)\cdot\cp{\frac{d-1}{d}}}^{m} \cdot \frac{1}{m} \notag\\
    & =\frac{2p}{d(d-1)} \cdot \frac{1}{1-\frac{(1-p)(d-1)}{d}} + \frac{2p}{(1-p)(d-1)} \cdot -\ln \cp{1-(1-p)\cdot \frac{d-1}{d}} \notag\\
    & = \frac{1}{(d-1)^2} + \frac{2\ln\cp{\frac{d}{2}}}{(d-1)(d-2)} \label{eq:gammaj_negative_mu2}
\end{align}

\newpage 
\subsection{Proof of Proposition~\ref{prop:expressivity}}\label{app:proof_prop_expressivity}
\begin{proof}
    Since \(\beta_s = 2\) and \(\beta_j = 0\) for every \(j\ne s\), we have \(f_{\attn, \readout}(w) = 2\gamma_s(w) - 1\). One can verify that \(f_{\attn, \readout}(w) = - 1\) for every \(w\notin A_s\), and hence, the negative examples do not contribute to the loss. For the positive examples $w\in A_s$, we separately consider words of length shorter or longer than $M=\frac{2d}{\varepsilon}$. Since length $|w|$ is a geometric random variable, $\mathbb{P}_{(w,y)\sim D} (|w|>M) = (1-\frac{1}{d-1})^M < e^{-2/\varepsilon} $. Since the squared loss remains bounded by 2 for the choice of $f_{\attn, \readout}(w) = 2\gamma_s(w)-1$, the contribution of long words to the risk $\mathcal{L}(\attn,\readout)$ is bounded by $2e^{-2/\varepsilon}<\varepsilon/2$.

    For short words, $|w|\leq M = 2d/\varepsilon$ with $C_s(w)\geq1$, the attention is bounded from below as $\gamma_s(w) \ge \frac{4d \cdot C_s(w)}{4d \cdot C_s(w) + \varepsilon^2 (m-C_s(w))}\geq \frac{4d}{4d + 2d\varepsilon} = \frac{2}{2 + \varepsilon}$, and from above since \(\gamma_s(w) \le 1\) due to softmax. Hence, the loss for any short word $w\in A_s$ is bounded as \( \frac12 \cp{f_{\attn,\readout}(w) - 1}^2 \leq \frac{2\varepsilon^2}{(2+\varepsilon)^2}<\frac{\varepsilon}{2}\). Hence, short words also contribute to at most $\varepsilon/2$ risk, and overall \(\mathcal{L}(\attn,\readout) \leq \varepsilon\).
\end{proof}

\subsection{Proof of Proposition~\ref{prop:attn_only}}\label{app:proof_prop_attn_only}
\begin{proof}
    Let the target index be \(s = \arg\min_{i\in[d]} \tilde\beta_i\) and discuss the loss based on its value. When \(\tilde\beta_s \ge 1\), it follows that \(\hat y(w) = \sum_{i=1}^d \tilde\beta_i \gamma_i(w) -1 \ge \min_{i\in[d]} \tilde\beta_i -1 \ge 0\) for any input \(w\) and any attention weights \(\attn\). In this case, the expected risk is lower bounded by \(\prob(y=-1)\times \frac 12 \times (0+1)^2 = \frac 14\).

    When \(\tilde\beta_s < 1\), we lower bound the risk incurred on negative samples \(z\) of length \(m\) together with their positive counterparts. For a given \(z\) of length \(m\), consider all positive samples \(z'\) obtained by inserting \(n\ge 1\) copies of \(s\) into \(z\). write \(R(d):=\frac{d^2-d}{d^2-2d+2}>1\), we first compute the ratio of the two probability masses:
    \begin{align*}
        \frac{\sum_{z'} \rho(z')}{\rho(z)} & = \frac{\sum_{n\ge 1} \binom{m+n}{n}(1-p)^{m+n-1} p \cp{\frac 1d}^{m+n}}{(1-p)^{m-1} p \cp{\frac 1d}^{m}} \\
        & = \sum_{n=0}^\infty \binom{m+n}{n}(1-p)^n\cp{\frac 1d}^n- 1 \\
        & = R(d)^{m+1}-1 \ge (m+1)(R(d)-1)
    \end{align*}
    Therefore, for each negative sample \(z\) whose length satisfies \(m \ge \frac{d+1}{2}\), it holds that \(\frac{\sum_{z'} \rho(z')}{\rho(z)} \ge \frac 12\).

    For each negative sample \(z\), write the output as \(\hat y(z) = G(z) -1 = \sum_{i\ne s} \tilde\beta_i \gamma_i(z)-1\), then for each of the positive counterpart \(z'\), we have \(\hat y(z') = \tilde\beta_s \gamma_s(z') + G(z) \cdot (1-\gamma_s(z')) - 1\). Now we distinguish two cases:
    \begin{itemize}
        \item \(G(z)\ge 1\): then \(\hat y(z)\ge 0\), and the error on the negative sample is \(\frac 12 (\hat y(z)+1)^2 = \frac 12 G(z)^2 \ge \frac 12\).
        \item If \(G(z)< 1\): since \(\tilde\beta_s<1\), the convex combination \(\tilde \beta_s \gamma_s(z') + G(z) \cdot (1-\gamma_s(z'))<1\) for every \(\gamma_s(z')\in[0,1]\), so \(\hat y(z')<0\), giving \(\frac12(\hat y(z')-1)^2 > \frac12\) on every positive counterpart, regardless of \(n\) or \(\attn\).
    \end{itemize}
    In either case, the probability mass of samples with an error greater than \(\frac 12\) is at least \(\min\{\rho(z),\sum_{z'}\rho(z')\}\). Summing over the window where the negative samples has length \(m\ge \lceil\frac{d+1}{2}\rceil\), we get:
    \begin{align*}
        \mathcal L(\attn,\tilde\readout) \ge \sum_{\substack{|z| \ge \lceil\frac{d+1}{2}\rceil \\ \chi_s(z) = -1}} \frac{1}{2} \cdot \min\{\rho(z),\sum_{z'}\rho(z')\} \ge \sum_{\substack{|z| \ge \lceil\frac{d+1}{2}\rceil \\ \chi_s(z) = -1}} \frac{1}{2} \cdot \frac 12 \cdot \rho(z) = \frac18\left(1-\frac2d\right)^{\lceil\frac{d+1}{2}\rceil-1} \ge \frac{1}{32},
    \end{align*}
    where the last inequality follows from the monotonicity of the bound on the even and odd sequences of \(d\), minimized at \(d=4\) and \(d=3\) respectively.
\end{proof}

\subsection{Proof of Proposition~\ref{prop:output_only}}\label{app:proof_prop_output_only}
\begin{proof}
    We consider the target symbol as \(s = \arg\min_{i\in[d]} \tilde\alpha_i\).
    First, consider any \(\beta_s \ge d + 1\), we can lower bound the expected risk by summing up the error on all samples made from only the set symbols:
    \begin{align*}
        \mathcal L_s(\tilde\attn, \readout)  \ge \frac 12 (\beta_s -2)^2 \cdot \sum_{m=1}^{\infty} (1-p)^{m-1} p \cp{\frac 1d}^m \ge \frac{d^2}{8\cb{(d-1)^2+1}} \ge \frac 18.
    \end{align*}
    Otherwise, when \(\beta_s < d +1\), we consider the prediction for every negative sample \(z\) whose length satisfies \(m=|z| > d\) and its corresponding positive counterparts where we insert one \(s\) into \(z\). We first compute the ratio between the two masses, which gives:
    \begin{equation}\label{eq:mass_ratio}
        \frac{\sum_{z'}\rho(z')}{\rho(z)} = \frac{\binom{m+1}{1} (1-p)^{m}p \cp{\frac 1d}^{m+1}}{(1-p)^{m-1}p \cp{\frac 1d}^{m}} = \frac{(m+1)(1-p)}{d} \ge \frac 12.
    \end{equation}
    For each negative sample \(z\), write the output as \(\hat y(z) = G(z) -1 = \sum_{i\ne s} \beta_i \gamma_i(z)-1\). And for each counterpart \(z'\), \(\hat y(z') = G(z)\cdot (1-\gamma_s(z')) + \beta_s \cdot \gamma_s(z')-1\). Now we distinguish two cases:
    \begin{itemize}
        \item \(G(z) > \frac 12\): then \(\hat y(z) > -\frac 12\), and the error on the negative sample is at least \(\frac 18\).
        \item \(G(z) \le \frac 12\): since \(\exp(\alpha_s) < \exp(\alpha_i)\) for any \(i\ne s\), we have \(\gamma_s(z') \le \frac{1}{1+d}\), and \(\hat y(z') \le \frac 12\), therefore the error on each \(z'\) is at least \(\frac 18\).
    \end{itemize}
    In either case, the probability mass of samples with an error greater than \(\frac{1}{8}\) is at least \(\min \{\rho(z), \; \sum_{z'}\rho(z')\}\). Summing over the window where the negative samples has length \(m \ge d + 1\), we have:
    \begin{align*}
        \mathcal L(\tilde\attn,\readout)  \ge \sum_{\substack{|z| \ge d+1 \\ \chi_s(z) = -1}} \frac{1}{8} \cdot \min\{\rho(z),\sum_{z'}\rho(z')\}\ge \sum_{\substack{|z| \ge d+1 \\ \chi_s(z) = -1}} \frac{1}{16} \cdot \rho(z) = \frac1{16}\cdot\frac12\cp{1-\frac2d}^{d} = \frac{1}{32}\cp{1-\frac2d}^{d} \ge \frac{1}{864}.
    \end{align*}
\end{proof}

\subsection{Proof of Proposition~\ref{prop:output_only_multi_heads}}\label{app:proof_output_only_multi_heads}
We begin by selecting the target index that receives a low attention score. We first prove that such an index exists:

\begin{lemma}[Existence of Index with Low Attention]\label{lem:pigeon_hole}
    For any fixed attention weights \(\attn^{(1)}, \dots, \attn^{(H)} \in \mathbb R^d\), there exists \(s\in[d]\), such that \(\displaystyle \sum_{h=1}^H \E\cb{\cp{\gamma_s^{(h)}(w)}^2} \le \frac{H}{d}\).
\end{lemma}
\begin{proof}
    By the definition of softmax, for each string \(w\) and each head \(h\), \(\gamma_i^{(h)}(w) \in [0,1]\) for every \(i\in[d]\), hence:
    \[
        \sum_{i=1}^d \cp{\gamma_i^{(h)}(w)}^2 \le \cp{\sum_{i=1}^d \gamma_i^{(h)}(w)}^2 = 1.
    \]
    Taking expectation over \(w\) and summing over all heads gives \(\sum_{h=1}^H \sum_{i=1}^{d} \E\cb{\cp{\gamma_i^{(h)}(w)}^2} \le H\). Therefore, there exists an index \(s\in[d]\) such that \(\sum_{h=1}^H \E\cb{\cp{\gamma_s^{(h)}(w)}^2} \le \frac{H}{d}\).
\end{proof}
We choose \(s\) that satisfies Lemma~\ref{lem:pigeon_hole} as the target index. Consider any negative string \(z\) whose length \(m\) satisfies \(\lceil \frac{d}{2} \rceil \le m \le d-1\), it follows that \(\rho(z) = (1-p)^{m-1}p\cp{\frac 1d}^m\). Let \(\mathcal P(z)\) be the set of positive strings which after removing all \(s\) will result in \(z\), we have \(\rho(\mathcal P(z)) = \sum_{n=1}^{\infty} \binom{m+n}{n} (1-p)^{m+n-1}p \cdot \cp{\frac{1}{d}}^{m+n}\). Further, we can consider the ratio between the two masses, which simplifies into:
\begin{align*}
    \frac{\rho(\mathcal P(z))}{\rho(z)} & = \sum_{n=1}^{\infty}\binom{m+n}{n} (1-p)^n \cp{\frac{1}{d}}^n = \cp{\frac{d-2}{d^2-2d+2}+1}^{m+1}- 1 \ge \frac{1}{2}.
\end{align*}
For any \(z'\in\mathcal P(z)\), we can write the output as \(\displaystyle f(z') = \sum_{h=1}^{H} \cb{\cp{1-\gamma_s^{(h)}(z')} g^{(h)}(z) + \beta_s^{(h)} \cdot \gamma_s^{(h)}(z')} - 1\), where \(g^{(h)}(z) = \sum_{i\ne s}\beta^{(h)}_i \gamma^{(h)}_i(z)\). Since \(|g^{(h)}(z)| \le \|\beta^{(h)}\|_2 \|\gamma^{(h)}(z)\|_2 \le \|\beta^{(h)}\|_2\) and \(\left|\beta_s^{(h)}\right| \le \left\|\beta^{(h)}\right\|_2\), we have:
\[
    |f(z') - f(z)| = \left|\sum_{h=1}^H \gamma_s^{(h)}(z') \cp{\beta_s^{(h)} - g^{(h)}(z)}\right|
    \le 2\sum_{h=1}^H \gamma_s^{(h)}(z') \|\beta^{(h)}\|_2
    \le 2\|\mathcal B\|_F \cp{\sum_{h=1}^H \cp{\gamma_s^{(h)}(z')}^2}^{1/2},
\]
where the last step is from Cauchy-Schwarz inequality over the \(H\) heads.

For each negative string \(z\) and each \(z' \in \mathcal P(z)\), we define a coupling weight \(\nu(z,z') := \min\cp{\rho(z), \rho(\mathcal P(z))} \cdot \frac{\rho(z')}{\rho(\mathcal P(z))}\). First, it directly follows that \(\nu(z, z') \le \rho(z')\). Second, \(\sum_{z' \in \mathcal P(z)} \nu(z, z') < \rho(z)\). Since \(\frac{\rho(\mathcal P(z))}{\rho(z)} \ge \frac12\) in the length window we considered for \(z\), it holds that \(\min(\rho(z),\rho(\mathcal P(z))) \ge \frac12 \rho(z)\), hence the total coupling mass is lower bounded by:
 \begin{align*} 
    \sum_{\lceil d/2 \rceil\le|z|\le d-1} \sum_{z'\in\mathcal P(z)} \nu(z,z') & = \sum_z \min(\rho(z),\rho(\mathcal P(z))) \ge \frac12 \sum_{\lceil d/2 \rceil\le|z|\le d-1} \rho(z) = \frac{1}{2}\sum_{m=\lceil d/2 \rceil}^d (1-p)^{m-1}p\cp{\frac{d-1}{d}}^m \\
    & = \frac14\cb{\cp{1-\frac2d}^{\lceil d/2\rceil - 1} - \cp{1-\frac2d}^{d-1}} := M
 \end{align*}
We write \(x = 1-\frac{2}{d}\). For any \(0<t<1\), we have \(\exp\cp{-\frac{t}{1-t}} \le 1-t \le \exp(-t)\). Plug in \(t=2/d\), we get \(x\ge \exp\cp{-\frac{2}{d-2}}\). Since \(\lceil d/2 \rceil - 1 \le \frac d 2\), we get $2\cdot \frac{\lceil d/2 \rceil - 1}{d-2} \le \frac{d}{d-2} \le 3$ for all $d\ge3$, so \(x^{\lceil d/2 \rceil - 1} \ge \exp\cp{-2\frac{\lceil d/2 \rceil - 1}{d-2}} \ge \exp(-3)\). Using the other side of the inequality, since \(x\le \exp\cp{-\frac{2}{d}}\) and \(\lfloor d/2\rfloor \ge \frac{d-1}{2}\), we have \(x^{\lfloor d/2\rfloor} \le \exp\cp{-2\frac{\lfloor d/2\rfloor}{d}} \le \exp\cp{-\frac{d-1}{d}} \le \exp\cp{-\frac23}\), which gives us \(1-x^{\lfloor d/2\rfloor} \ge 1-\exp\cp{-\frac23}\). Combining both sides, we have \(M \ge \frac14\exp(-3)\cp{1-\exp\cp{-\frac23}}>0\), which is an absolute constant independent of \(d\).

Define \(\E_{\mathrm{pair}}\cb{\cdot} := \frac{1}{M} \sum_{z,z'} \nu(z,z') \times(\cdot)\). Since $\nu(z,z') \le \rho(z')$, we have:
\[
    \E_{\mathrm{pair}}\cb{\sum_{h=1}^H \cp{\gamma_s^{(h)}(z')}^2}
    \le \frac{1}{M} \sum_{h=1}^H \E_\rho\cb{\cp{\gamma_s^{(h)}(w)}^2}
    \le \frac{H}{dM},
\]
using Lemma~\ref{lem:pigeon_hole} in the last step. Finally, by Jensen's inequality, we have:
\begin{equation}\label{eq:pos_neg_gap}
    \E_{\mathrm{pair}} |f(z')-f(z)| \le 2 \|\mathcal B\|_F \sqrt{\frac{H}{dM}}.
\end{equation}

For every pair \((z,z')\), \(y(z) = -1\) and \(y(z') = 1\). Denote $e_z := f(z)+1$ and $e_{z'} := f(z')-1$, we have \(e_z^2 + e_{z'}^2 \ge \frac12\cp{2 - |f(z')-f(z)|}_+^2\). Since $u \mapsto (2-u)_+^2$ is convex, combining Jensen's inequality and Eq.~\ref{eq:pos_neg_gap}, we have:
\[
    \E_{\mathrm{pair}}\cb{e_z^2+e_{z'}^2} \ge \frac12\cp{2 - 2\|\mathcal B\|_F\sqrt{\frac{H}{dM}}}_+^2.
\]

Because $\displaystyle \sum_{z'\in\mathcal P(z)}\nu(z,z') \le \rho(z)$ and $\nu(z,z') \le \rho(z')$ as established above, the weighted sum $\sum_{z,z'} \nu(z,z')(e_z^2+e_{z'}^2)$ is bounded above by the corresponding disjoint contribution to $\E_\rho\cb{(f(w)-y(w))^2}$, hence we can bound the risk by:
\[
    \mathcal L_{D}(f) = \frac12 \cdot \E_\rho\cb{(f(w)-y(w))^2} \ge \frac12 M \cdot \E_{\mathrm{pair}}\cb{e_z^2+e_{z'}^2} \ge \frac{M}{4}\cp{2-2\|\mathcal B\|_F\sqrt{\frac{H}{dM}}}_+^2.
\]
Equivalently, when the number of heads \(H \le \frac{dM}{4\|\mathcal B\|_F^2}\), the loss is lower bounded by \(\frac M 4\). Since we know the output norm is bounded by \(\|\mathcal B\|^2_F \le \sqrt d\), we have that when \(H \le \frac{M}{4}\sqrt d\), the loss is lower bounded by \(\frac{M}{4}\).
\subsection{Proof of Theorem~\ref{thm:attention-first-fails}}\label{app:proof_attention_first_fails}
\begin{proof}
    First, we prove that for whatever \(\attn'\) we obtained after the first step, if \(\beta_s > \frac{d}{4}\) then the population risk is already lower bounded by a constant. Here \(\readout \in \mathbb R^d\) is an arbitrary vector. Consider all positive samples with only \(s\), then:
    \[
        \mathcal L(\attn', \readout) \ge \sum_{m=1}^{\infty} \frac{1}{2} \cp{\frac d 4 - 2}^2 \cdot (1-p)^{m-1} p \cp{\frac 1 d}^m = \frac{(d-8)^2}{32\cb{(d-1)^2+1}} \ge \frac{1}{2080}.
    \]
    Hence in the rest of the analysis, we will only consider the case where \(\beta_s \le \frac{d}{4}\). Now we prove that after the first gradient step, we can find an index \(s\) such that after the update, \(\alpha'_s \le 0\). Consider the following sum of the gradients:
    \begin{align*}
        \sum_{i=1}^d \fracpartial{\mathcal L_{i}}{\alpha_i} 
        & = \sum_{i=1}^d \E\cb{\cp{\hat y(w) - \chi_i(w)}\cdot \gamma_i \cp{\beta_i - \sum_{j=1}^d \beta_j \gamma_j}}\\
        & = \sum_{i=1}^d \E\cb{\cp{\sum_{j=1}^d \beta_j \gamma_j - 1 - 2\cdot \mathbf 1\{C_i(w)>0\} +1}\gamma_i \cp{\beta_i - \sum_{j=1}^d \beta_j \gamma_j}} \\
        & = \E\cb{\sum_{i=1}^{d} \beta_i \gamma_i \cdot \cp{\sum_{j=1}^d \beta_j \gamma_j} -2 \sum_{i=1}^d \beta_i\gamma_i - \sum_{i=1}^d \gamma_i \cp{\sum_{j=1}^d \beta_j \gamma_j}^2 + 2\sum_{i=1}^d \gamma_i \cp{\sum_{j=1}^d \beta_j \gamma_j}} \\
        & = \E\cb{\cp{\sum_{i=1}^d \beta_i \gamma_i}^2 - 2\sum_{i=1}^d \beta_i \gamma_i - \cp{\sum_{j=1}^d \beta_j \gamma_j}^2 + 2\sum_{j=1}^d \beta_j \gamma_j} = 0
    \end{align*}
    So there exists \(s\in[d]\) such that \(\fracpartial{\mathcal L}{\alpha_s} \ge 0\) and \(\alpha'_s \le 0\). Consider the sum of gradients of \(\mathcal L_s\) across all indices:
    \begin{align*}
        \sum_{i=1}^d \fracpartial{\mathcal L_s}{\alpha_i} = \sum_{i=1}^d \E\cb{\cp{\hat y - y} \cdot \gamma_i \cdot \cp{\beta_i - \sum_{j=1}^d\beta_j\gamma_j}} = \E\cb{\cp{\hat y - y} \cdot \sum_{i=1}^d \gamma_i \cdot \cp{\beta_i - \sum_{j=1}^d\beta_j\gamma_j}} = 0,
    \end{align*}
    hence it follows \(\sum_{i=1}^d \alpha'_i = 0\) regardless of the learning rate. Therefore, \(\displaystyle \sum_{i=1}^{d} \exp(\alpha'_i) \ge d \cdot \cp{\prod_{i=1}^d \exp(\alpha_i')}^{\frac 1d}=d\). At the same time, \(\alpha'_s \le 0\), which results in \(\exp(\alpha'_s) \le 1\). We will shorthand \(a_i:=\exp(\alpha'_i)\) in the remaining analysis.
    
    Now we consider a negative string \(z\) with length \(m\), where \(m \ge d\), and the set of its positive counterparts \(Z\) where each \(z'\in Z\) is obtained by inserting exactly one symbol \(s\) into \(z\). From \eqref{eq:mass_ratio} we have \(\frac{\rho(Z)}{\rho(z)} \ge \frac 12\). 
    
    For each counter part \(z'\), write \(\displaystyle \gamma_s(z') = \frac{a_s}{\sum_{i\ne s} C_i(z) a_i + a_s}\). Consider the two following cases:
    \begin{itemize}
        \item \(\displaystyle\max_{i\ne s} a_i \ge \frac d 2\). In this case, the probability that the maximizer \(i\) appears in string \(z\) is \(1-\cp{1-\frac{1}{d-1}}^{m} \ge 1-\cp{1-\frac{1}{d-1}}^{\frac d 2}\ge 1-\exp\cp{-\frac12} \ge \frac 18\), under which case we have \(\displaystyle \gamma_s(z') \le \frac{1}{\frac d 2 + 1} = \frac{2}{d+2} < \frac{4}{d}\).
        \item \(\displaystyle\max_{i\ne s} a_i < \frac d 2\). In this case, consider the r.v. \(\xi\sim\text{Unif}(\{a_i : i\ne s\})\), and we rewrite \(\displaystyle \sum_{i\ne s} C_i(z) a_i = \sum_{t=1}^m \xi_t:=S\). We further derive the first and the second moments of \(\xi\):
        \begin{align*}
            \E[\xi] = \frac{1}{d-1}\sum_{i\ne s} a_i \ge \frac{1}{d-1} \cdot (d-1)=1, \quad \E[\xi^2] \le \max_{i \ne s} a_i \E[\xi] \le \frac d 2 \cdot \E[\xi]. \\
        \end{align*}
        Hence, the first and the second moments of \(S\) can be written as:
        \[\E[S] = m\E[\xi], \quad \E[S^2] = m\E[\xi^2] + m(m-1)\E[\xi]^2 \le m\cdot \frac d 2 \cdot \E[\xi] + m(m-1) \E[\xi]^2. \]
        Using the Paley-Zygmund inequality, we have:
        \[
            \prob\cp{S \ge \frac 12 \E[S]} \ge \frac{\E[S]^2}{4\E[S^2]} \ge \frac{m^2\E[\xi]^2}{4m \frac d 2\E[\xi] + 4 m(m-1) \E[\xi]^2} = \frac{m\E[\xi]}{2d + 4(m-1)\E[\xi]} \ge \frac{m}{2d + 4(m-1)} \ge \frac{1}{2\cdot \frac{d}{m} + 4} \ge \frac 18.
        \]
        Since \(\E[S] \ge m\), we have \(\prob\cp{S \ge \frac m 2} \ge \frac{1}{8}\). Under which case we have \(\displaystyle \gamma_s(z') \le \frac{a_s}{a_s + \frac{m}{2}} \le \frac{4}{d+4} < \frac{4}{d}\).
    \end{itemize}
    For each $m\ge d$, let \(G_m:=\left\{ z\in(\Sigma\setminus\{s\})^m: \gamma_s(z')<\frac4d \text{ for every }z'\in Z(z) \right\}.\) 
    
    Conditional on $|z|=m$ and $\chi_s(z)=-1$, the symbols of $z$ are i.i.d.\ uniform on $\Sigma\setminus\{s\}$. The two cases before show that \(\prob(z\in G_m \given |z|=m, y=-1)\ge\frac18\). Fix any final output vector $\beta$ with $\beta_s\le d/4$, and write \(G(z):=\sum_{i\ne s}\beta_i\gamma_i(z).\) For every $z\in G_m$, either $G(z)\ge 1/2$, in which case \(\ell(z)=\frac12G(z)^2\ge\frac18\), or $G(z)<1/2$. In the latter case, for every $z'\in Z(z)$, we have \(f_{\alpha',\beta}(z')-1 =G(z)(1-\gamma_s(z'))+\beta_s\gamma_s(z')-2 <\frac12+1-2=-\frac12\), so $\ell(z')\ge 1/8$. 
    
    Consequently, it follows that \(\rho(z)\ell(z)+ \sum_{z'\in Z(z)}\rho(z')\ell(z') \ge \frac18\min\{\rho(z),\rho(Z(z))\} \ge\frac1{16}\rho(z)\). The sets $\{z\}\cup Z(z)$ are disjoint across negative strings $z$. Hence, with $G:=\bigcup_{m\ge d}G_m$, \(\mathcal L_s(\alpha',\beta) \ge \frac1{16}\rho(G) \ge \frac1{128} \prob_{\rho}(y(w)=-1,\ |w|\ge d) = \frac{1}{256}\cp{1-\frac2d}^{d-1} \ge \frac{1}{6912}\).
    
    For the last inequality, when $d\ge9$, \((d-1)\ln\left(1-\frac2d\right) \ge-\frac{2(d-1)}{d-2} \ge-\frac{16}{7}>-3\), and $e^{-3}>1/27$. Combining this bound with the earlier case $\beta_s>d/4$, which gives $\mathcal L_s(\alpha',\beta)\ge1/2080$, proves the theorem with $c_3=1/6912$.
\end{proof}

\subsection{Proof of Theorem 2}\label{app:proof_thm_4}
\begin{proof}
Throughout this proof, $\readout'$ denotes the output weights after Phase~1, $\attn'$ the attention weights after Phase 2, and $\readout(t)$ the output weights after $t \geq 0$ steps of Phase 3, so that $\readout(0) = \readout'$ and $\readout'' = \readout(T)$. We first prove the following lemma.
\begin{lemma}[Induction on Symmetry]\label{lem:symmetry}
    For any target index \(s\in[d]\), if at the current step we have \(\beta_i = \beta_j\) and \(\alpha_i = \alpha_j\) for every pair of background indices \(i, j \ne s\), then after one step of population GD, the updated parameters satisfy \(\alpha_i' = \alpha_j'\), and \(\beta_i' = \beta_j'\).
\end{lemma}

\begin{proof}
    For any fixed target index \(s\), let \(\tau_{i,j} := (i\ j) \in \mathrm{Sym}([d])\) denote the transposition swapping \(i,j\ne s\) while keeping other indices fixed. For any string \(w = (x_1, x_2\cdots x_m)\in\Sigma^m\), we write \(
    \tau_{i,j}(w) = \tau_{i,j}(x_1)\tau_{i,j}(x_2)\cdots\tau_{i,j}(x_m)\).
    It follows that \(\rho(w) = \rho(\tau_{i, j}(w))\) for every \(w\in\Sigma^+\). Moreover, since \(i,j\ne s\), swapping \(v_i\) and \(v_j\) will not change the acceptance of the string, i.e., \(y(w) = y(\tau_{i,j}(w))\).

    Next, define the corresponding action on parameters: for \((\attn,\readout)\in\mathbb R^d\times\mathbb R^d\), let \(\tau_{i,j}(\attn,\readout)\) denote the parameter vector obtained by swapping \(\alpha_i\) with \(\alpha_j\) and \(\beta_i\) with \(\beta_j\), leaving all other coordinates fixed. It follows that \(\displaystyle f_{\attn,\readout}(w) = f_{\tau_{i,j}(\attn,\readout)}(\tau_{i,j}(w))\) for every  \(w\in\Sigma^+\).

    Combining these three facts, the population loss is invariant under this joint relabeling:
    \begin{align*}
        \mathcal L(\attn,\readout) & = \E_{w\sim\rho}\cb{\cp{f_{\attn,\readout}(w)-y(w)}^2} = \E_{w\sim\rho}\cb{\cp{f_{\tau_{i,j}(\attn,\readout)}(\tau_{i,j}(w))-y(\tau_{i,j}(w))}^2} \\
        & = \E_{w'\sim\rho}\cb{\cp{f_{\tau_{i,j}(\attn,\readout)}(w')-y(w')}^2} = \mathcal L\cp{\tau_{i,j}(\attn,\readout)}.
    \end{align*}

    Therefore, for every pair of weights \((\attn,\readout)\), it holds that \(\fracpartial{\mathcal L}{\alpha_i}\cp{\tau_{i,j}(\attn,\readout)} = \fracpartial{\mathcal L}{\alpha_j}(\attn,\readout), \; \fracpartial{\mathcal L}{\beta_i}\cp{\tau_{i,j}(\attn,\readout)} = \fracpartial{\mathcal L}{\beta_j}(\attn,\readout).\)
    Since the current parameters satisfy \(\alpha_i=\alpha_j\) and \(\beta_i=\beta_j\), then \(\tau_{i,j}(\attn,\readout) = (\attn,\readout)\), and the identities above rewrites as \(\fracpartial{\mathcal L}{\alpha_i}(\attn,\readout) = \fracpartial{\mathcal L}{\alpha_j}(\attn,\readout), \; \fracpartial{\mathcal L}{\beta_i}(\attn,\readout) = \fracpartial{\mathcal L}{\beta_j}(\attn,\readout)\).
    Plugging these equalities into the GD update rule we have
    \(\alpha_i' = \alpha_i - \eta\fracpartial{\mathcal L}{\alpha_i} = \alpha_j - \eta\fracpartial{\mathcal L}{\alpha_j} = \alpha_j'\). Same equality holds for \(\readout\) as well.
\end{proof}

With this lemma, and the symmetrical initialisation for \(\attn\) and \(\readout\), the parameter can be reduced to 4 scalars, \(\alpha_s\) and \(\beta_s\) for the target, and \(\alpha_j\), \(\beta_j\) as the common weight for the background indices. We also shorthand \(\gamma_i(w)\) as the aggregated softmax-attention score for symbol \(v_i\), note that sample-wise there is no symmetry for \(\gamma_i\) when \(i\ne s\), but on expectation level we have \(\E[\gamma_j]\) which is shared by all background indices \(j\ne s\). We also shorthand \(\Delta \beta = \beta_s - \beta_j\) and \(\bar\beta = \frac{1}{d-1}\sum_{i\ne s} \beta_i\).

\textbf{First Phase} In the first step, we calculate the gradient for \(\beta_s\) and \(\beta_j\) in closed forms:
\begin{align*}
    \fracpartial{\mathcal L}{\beta_s} & = \E\cb{(\hat y (w) - y(w))\gamma_s} \overset{(a)}{=} \E\cb{(-1-y(w))\gamma_s} = \prob(y=1) \cdot -2\cdot \E\cb{\gamma_s \given y=1} = -\frac{2}{d}. \\
    \fracpartial{\mathcal L}{\beta_j} & = \E\cb{(\hat y (w) - y(w))\gamma_j} = \E\cb{(-1-y(w))\gamma_j} = \prob(y=1) \cdot -2\cdot \E\cb{\gamma_j \given y=1} = \frac{1}{d-1} \cdot \cp{\frac 2 d -1}
\end{align*}
Hence, the updated parameter is \(\beta_s' = \frac{2\eta_1}{d}, \; \beta_j' = \frac{\eta_1}{d-1} \cdot \cp{1-\frac 2 d}\) for every \(j\ne s\). The separation of target and the background output weights grows from \(\Delta\beta = 0\) to \(\Delta\beta' = \frac{\eta_1}{d-1}\). We choose the learning rate of the first phase to be \(0< \eta_1 < \frac{d-1}{2}\) and therefore the separation will be controlled within a constant range \(0<\Delta\beta' < \frac{1}{2}\).

\textbf{Second Phase} In the second phase, we calculate the gradient for \(\alpha_s\) and \(\alpha_j\) in closed forms:
\begin{align*}
    \fracpartial{\mathcal L}{\alpha_s} & = \Delta\beta' \cdot \E\cb{(\hat y(w) - y(w))\cdot\gamma_s(1-\gamma_s)}\\
    \fracpartial{\mathcal L}{\alpha_j} & = -\frac{1}{d-1}\cdot \fracpartial{\mathcal L}{\alpha_s}
\end{align*}
Consider \(0< \eta_1 < \frac{d-1}{2}\), we have \(\beta_i' < 1\) for every \(i\in[d]\), therefore for positive words \(\hat y(w) = \sum_{i=1}^d \gamma_i\beta_i' - 1<0\) and \(f-y < -1\), hence we get an upper bound for the target gradient (which holds for negative words too since \(\gamma_s = 0\)):
\begin{align*}
    \fracpartial{\mathcal L}{\alpha_s} \le -\Delta \beta' \E[\gamma_s(1-\gamma_s)]
\end{align*}
At uniform attention, the moments in \eqref{eq:positive_mu1} and \eqref{eq:positive_mu2} give \(d\E[\gamma_s(1-\gamma_s)] = \frac{d-1}{d} \cp{1-\frac{\ln(d-1)}{d-2}} \ge \frac{2}{3}(1-\ln 2) > \frac{1}{20}\) for every \(d\ge 3\). Hence we have \(\fracpartial{\mathcal L}{\alpha_s} \le -\frac{\eta_1}{20d(d-1)}\) and \(\fracpartial{\mathcal L}{\alpha_j} > \frac{\eta_1}{20d(d-1)^2}\). Combining both bounds, the attention gap after the second step satisfies
\[
    \Delta\alpha' := \alpha_s' - \alpha_j'
    = -\frac{d}{d-1}\,\eta_2\,\fracpartial{\mathcal L}{\alpha_s}
    \;\geq\; \frac{\eta_1\eta_2}{20(d-1)^2}
    \;\geq\; \ln\frac{8d}{\varepsilon}.
\]
For a positive string $w$ with $|w| = m$ and $C_s(w) = n \geq 1$,
\[
    1 - \gamma_s'(w) = \frac{m-n}{n\exp(\Delta\alpha') + m - n}
    \;\leq\; (m-1)\exp(-\Delta\alpha').
\]
Since $\E\cb{m - 1 \given y = 1} \leq 2\,\E[m-1] = 2(d-2)$, we obtain
\[
    \E\cb{1 - \gamma_s' \given y = 1}
    \;\leq\; \delta := 2d\exp(-\Delta\alpha')
    \;\leq\; \frac{\varepsilon}{4}.
\]

\textbf{Third Phase} During the third step, the attention is fixed at $\attn'$, so $\mathcal L(\attn', \readout)$ is a convex quadratic function w.r.t. \(\readout\) with Hessian $H = \E[\boldsymbol\gamma'\boldsymbol\gamma'^\top]$. Let $\mu_k' := \E\cb{(\gamma_s')^k \given y = 1}$ for \(k=1,2,3\), so that $1 - \delta \leq \mu_1'$ and $(\mu_1')^2 \leq \mu_2' \leq 1$. Let $S := \mathrm{span}\{e_s, \mathbf 1_{-s}\}$ and $u := \mathbf 1_{-s}/\sqrt{d-1}$.

On $S$, the risk is minimized at $\readout^{\star}$
with:
\[
    \beta_{s}^{\star} = \frac{2\mu'_1+2\mu'_2-2(\mu'_1)^2}{2\mu'_2 - (\mu'_1)^2} , \quad \beta_{j}^{\star} = \frac{2\mu'_2-2(\mu'_1)^2}{2\mu'_2 - (\mu'_1)^2}.
\]
And the expected risk at this point will be:
\begin{align*}
    \mathcal L(\attn', \readout^{\star})& = \frac 12 \cdot \E\cb{\frac 12 \cp{\beta_s^{\star} \gamma'_s + \sum_{j\ne s} \beta^{\star}_j \gamma'_j -1 -1}^2 \given y = 1} + \frac 12 \cdot \E\cb{\frac 12 \cp{\sum_{j\ne s} \beta^{\star}_j\gamma'_j-1+1}^2 \given y = -1} \\
    & = \frac 14 \cdot \cp{\E\cb{\cp{\Delta\beta^{\star} \gamma'_s + \beta^{\star}_j -2}^2 \given y = 1} + \cp{\beta^{\star}_j}^2} \\
    & = \frac 14 \cp{\Delta\beta^{\star}}^2 \mu_2' + \frac 12 \Delta\beta^{\star} (\beta_j^{\star}-2) \mu_1' + \frac 14(\beta_j^{\star}-2)^2 + \frac 14(\beta_j^{\star})^2 = \frac{\mu_2'-(\mu_1')^2}{2\mu_2'-(\mu_1')^2} < 2\delta
\end{align*}
Moreover, we have $\beta_j^{\star} = 2\mathcal L(\attn',
\readout^{\star}) \in [0, 4\delta]$ and
$0 \leq \beta_s^{\star} \leq 4\delta + \frac{2}{\mu_1'} < 4$.


At any step in the third phase, the symmetry is preserved for the \(\attn'\), \(\readout\), and the \(\E[\boldsymbol{\gamma'}]\) among all indices \(j\ne s\), as a result of Lemma~\ref{lem:symmetry}. 
Since $\readout^{\star}$ minimizes the risk over $S$, its gradient is orthogonal to $S$ and hence zero. Therefore, for every $\readout \in \mathbb R^d$,
\[
    \mathcal L(\attn', \readout) - \mathcal L(\attn', \readout^{\star}) = \tfrac12 (\readout - \readout^{\star})^\top H (\readout - \readout^{\star}),
\]
and the GD iterates with $\eta_3 = 1$ satisfy $\readout(t+1) - \readout^{\star} = (I - H)(\readout(t) - \readout^{\star})$. By Lemma~\ref{lem:symmetry}, all iterates lie in $S$, so only the restriction $\tilde H$ of $H$ to $S$ matters. Take an orthonormal basis \((e_s, \frac{\mathbf1_{-s}}{\sqrt{d-1}},\ \xi_1,\dots,\xi_{d-2})\) of \(\mathbb R^d\), where \(\xi_1,\dots,\xi_{d-2}\) span the subspace \(\{e_s,\mathbf 1_{-s}\}^\perp\). We can then block-diagonalize \(H\) in this basis as \(H = \tilde H \oplus aI_{d-2}\) with \(a=\Theta\cp{\frac{\ln d}{d^2}}\), where \(\tilde H\in\mathbb R^{2\times2}\) acts on \(\mathrm{span}\{e_s,\mathbf1_{-s}\}\). Each component of \(\tilde H\) has closed forms:
\begin{align*}
    \tilde H_{11} &= \E[(\gamma_s')^2] = \tfrac12\mu_2' \in \cb{\tfrac12(1-\delta)^2,\ \tfrac12}, \\
    \tilde H_{22} &= \frac{\E[(1-\gamma_s')^2]}{d-1} \in \cb{\frac{1}{2(d-1)},\ \frac{1+\delta}{2(d-1)}}, \\
    \tilde H_{12} = \tilde H_{21}&= \frac{\E[\gamma_s'(1-\gamma_s')]}{\sqrt{d-1}} \in \cb{0,\ \frac{\delta}{2\sqrt{d-1}}},
\end{align*}
where we used that $\E[1 - \gamma_s'\given y=-1] = 1$ and $\E\cb{1-\gamma_s' \given y=1} \leq \delta$. Hence $\mathrm{tr}(\tilde H) \leq \tfrac12 + \tfrac{1+\delta}{4} \leq 1$ for $d \geq 3$, and
\[
    \det(\tilde H) \geq \frac{(1-\delta)^2 - \delta^2}{4(d-1)} = \frac{1-2\delta}{4(d-1)}, \qquad \lambda_{\min}(\tilde H) \geq \frac{\det(\tilde H)}{\mathrm{tr}(\tilde H)} \geq \frac{1}{8d},
\]
using $\delta \leq \tfrac14$. Since also $\lambda_{\max}(\tilde H) \leq 1$, each step contracts the error by a factor of at most $1 - \frac{1}{8d}$.

At the beginning of Phase 3, output starts from $\readout'$, with $\beta_s' \in (0, 1)$ and $\beta_j' \in (0, \tfrac12)$, so $\|\readout' - \readout^{\star}\|^2 \leq \|\readout^{\star}\|^2 + \|\readout'\|^2 \le 16 + (d-1) \leq 16d$. Therefore \(\displaystyle \mathcal L(\attn', \readout(T)) - \mathcal L(\attn', \readout^{\star}) \;\leq\; \tfrac12\,\|\readout(T) - \readout^{\star}\|^2 \;\leq\; 8d\cp{1 - \frac{1}{8d}}^{2T} \;\leq\; 8d\exp\cp{-\frac{T}{4d}} \;\leq\; \frac{\varepsilon}{2}\) when plugging in $T \geq 4d\ln\frac{16d}{\varepsilon}$. After \(T\) steps we have $\readout'':= \readout(T)$ and hence $\mathcal L(\attn', \readout'') \leq 2\delta + \frac{\varepsilon}{2} \leq \varepsilon$.
\end{proof}

\subsection{Proof of Corollary~\ref{thm:fbGD}}\label{app:proof_cor_1}
\begin{proof}
We use $(w^{(k)}, y^{(k)})_{k=1}^{N_1}$ denotes the first batch, and for any function $g(w)$ of a sample we write $\hat{g} := \frac{1}{N_1}\sum_{k=1}^{N_1} g(w^{(k)}, y^{(k)})$ for its average over this batch. In Phases 1 and 2, $\attn = \mathbf{0}_d$, so $\gamma_i(w) = \frac{C_i(w)}{|w|}$. 
 
With respect to the first draw of samples, we consider the following variables. For any sample $(w, y)$ and $i \in [d]$, we define \(Z_i := \mathbf{1}\{y = 1\}\,\gamma_i(w), \; V := \gamma_s(w)\cp{1 - \gamma_s(w)}\).
Note that none of these quantities depend on the trained parameters but their randomness come from the raw strings that are drawn.

\begin{lemma}[Concentration on the first batch]\label{lem:first_batch}
    For every $c > 0$, there exists a constant $C > 0$ such that, when $N_1 \geq C d \ln d$, with probability at least $1 - d^{-c}$, the following event $A$ holds :
    \begin{enumerate}
        \item $\left|\hat{Z_i} - \E[Z_i]\right| \leq \frac{1}{800 d}$ for every $i \in [d]$;
        \item $\hat{V} \geq \frac12 \E[V]$;
        \item $\hat{\gamma}_j \leq \frac{2}{d}$ for every
        $j \neq s$.
    \end{enumerate}
\end{lemma}
\begin{proof}
    All of $Z_i$, $V$, $\gamma_j$ take values in $[0, 1]$. We can bound their second moments by:
    \begin{align*}
        \E[Z_i^2] &= \frac 1 2\cdot \E[\gamma^2_i \given y=1] = \frac{1}{(d-1)^2} + \frac{\ln(d-1)}{(d-1)(d-2)} \le \frac{1+\ln(d-1)}{(d-1)^2}; \\
        \E[V^2] & \le \E[\gamma_s^2 \given y=1] \le \frac{2+2\ln d}{d^2}; \\
        \E[\gamma_j^2] & = \frac{1}{d^2} + \frac{(d-1)\ln(d-1)}{d^2(d-2)} \le \frac{1+2\ln d}{d^2}.
    \end{align*}
    Hence, we can use \(\sigma^2 = \frac{2+2\ln d}{(d-1)^2} \le \frac{12\ln d}{d^2}\) as a common upper bound for the three terms above. And we also have \(\E[V] = \E[\gamma_s] - \E[\gamma_s^2] = \frac1d - \frac{1}{d^2} - \frac{(d-1)\ln(d-1)}{d^2(d-2)} \ge \frac{1}{5d}\) for all \(d \ge 3\).
    By Bernstein's inequality, for each statement in the event \(A\), the failure probability of the each term can be bounded by:
    \begin{align*}
        \prob(|\hat Z_i - \E[Z_i]| > \tfrac{1}{800d}) & < 2\exp\cp{-\frac{N_1 \cp{\frac{1}{800d}}^2}{2 \cdot \frac{12 \ln d}{d^2} + \frac{1}{1200d}}} < 2\exp\cp{-\frac{N_1}{C_0 \ln d + C_1 d}}, \\
        \prob(\E[V] - \hat V > \tfrac 12 \E[V]) &< \exp\cp{-\frac{N_1 \cp{\frac{1}{10d}}^2}{2 \cdot \frac{12 \ln d}{d^2} + \frac{1}{15d}}} < \exp\cp{-\frac{N_1}{C_0' \ln d + C_1' d}} \\
        \prob(\hat \gamma_j -  \E[\gamma_j]> \tfrac{1}{d}) & < \exp\cp{-\frac{N_1 \cp{\frac{1}{d}}^2}{2 \cdot \frac{12 \ln d}{d^2} + \frac{2}{3d}}} < \exp\cp{-\frac{N_1}{C_0''' \ln d + C_1''' d}}
    \end{align*}
    Taking the union bound of all the \(2d+1\) statements, choosing
    $N_1 \geq C d \ln d$ with $C$ large enough (depending on $c$) makes the combined failure at most $\frac{d^{-c}}{2d+1}$.
\end{proof}
 
\textbf{First Phase.} When $\attn = \readout = \mathbf{0}_d$, the output is $f(w)=-1$ on every string, so $\fracpartial{\hat{\mathcal L}}{\beta_i} = -2\hat{Z_i}$ and the updated output weight becomes $\hat\beta_i' = 2\eta_1 \hat{Z_i} \geq 0$. By the moment of the attention scores specified in \eqref{eq:gammaj_positive_mu1}, \eqref{eq:gammaj_positive_mu2}, \eqref{eq:gammaj_negative_mu2}, and $\prob(y=-1) = \frac12$, we have \( \E[Z_s] = \frac1d, \; \E[Z_j] = \frac1d - \frac{1}{2(d-1)} = \frac{d-2}{2d(d-1)}\) for every \(j \ne s\).
With Lemma~\ref{lem:first_batch} and \(0 < \eta_1 \leq \sqrt d\), we have:
\begin{align}
    &\hat\beta_s' \leq \eta_1\cp{\frac2d + \frac{1}{400d}} \leq \frac{2.01}{\sqrt d} \leq \frac54 \label{eq:beta_s} \\
    &\hat\beta_s' - \max_{i\ne s}\hat\beta_i' \geq \eta_1\cp{\frac{1}{d-1} - \frac{1}{200d}} \geq \frac{199}{200}\cdot\frac{\eta_1}{d}, \label{eq:beta_sep} \\
    &\left|\hat\beta_j' - \hat\beta_{j'}'\right| \leq \frac{\eta_1}{200d}, \text{ for } j, j'\ne s; \;\; 0 \leq \hat\beta_j' \leq \frac{1.01\,\eta_1}{d} .\label{eq:beta_spread}
\end{align}

\textbf{Second Phase.} For a sample $(w, y)$ and $i \in [d]$, we define \(T_i := \cp{y - f(w)}\,\gamma_i \sum_{j \neq i} \cp{\hat\beta_i' - \hat\beta_j'}\gamma_j\), Then the empirical gradient w.r.t. \(\alpha_i\) can be written as $\fracpartial{\hat{\mathcal L}}{\alpha_i} (\mathbf{0}_d, \hat{\readout}') = -\hat{T_i}$, hence for every background symbol \(j\ne s\):
\begin{equation}\label{eq:gap_identity}
    \hat\alpha_s' - \hat\alpha_j' = \eta_2\cp{\hat{T_s} - \hat{T_j}}
\end{equation}
Since the gap and the output weights $\hat{\readout}'$ depend on the same batch of samples, we will not apply a concentration inequality directly to $T_i$. Instead, we bound $T_i$ sample by sample, in terms of $V$ and $\gamma_j$.
 
We first consider a lowe bound for \(T_s\). For negative samples, we have $y = -1$, then $\gamma_s = 0$ and $T_s = 0 = V$. If $y = 1$, then \(y - f(w) = 2 - \sum_{i=1}^d \hat\beta_i'\gamma_i \geq 2 - \hat\beta_s' \geq \tfrac34\) by \eqref{eq:beta_s}, and \(\sum_{i \neq s}(\hat\beta_s' - \hat\beta_i')\gamma_i \geq \frac{199}{200}\frac{\eta_1}{d}(1 - \gamma_s)\) by \eqref{eq:beta_sep}. Hence $T_s \geq \frac34\cdot\frac{199}{200}\cdot\frac{\eta_1}{d}\, V$ for every sample, regardless of positive or negative, and by Lemma~\ref{lem:first_batch}, when \(A\) happens:
\[
    \hat{T_s} \;\geq\; \frac34\cdot\frac{199}{200}\cdot \frac{\eta_1}{d}\cdot\frac{1}{10d} \ge \frac{7\,\eta_1}{100\,d^2}.
\]
 
Then we bound \(T_j\) from above. For positive samples, $y = 1$, then \(0 < y - f(w) \leq 2\), and
\[
    \sum_{i \neq j}\cp{\hat\beta_j' - \hat\beta_i'}\gamma_i = \cp{\hat\beta_j' - \hat\beta_s'}\gamma_s + \sum_{i \neq j, s}\cp{\hat\beta_j' - \hat\beta_i'}\gamma_i \le \frac{\eta_1}{200d},
\]
since the first term is negative and the second is bounded by
\eqref{eq:beta_spread}. For negative samples $y = -1$, then $\gamma_s = 0$, \(|y - f(w)| = \sum_{i \neq s}\hat\beta_i'\gamma_i \leq 2\), and \(\left|\sum_{i \neq j}(\hat\beta_j' - \hat\beta_i')\gamma_i\right| \le \frac{\eta_1}{200d}\). For both positive and negative samples, we have \(T_j \leq 2\gamma_j\cdot\frac{\eta_1}{200d}\), and by Lemma~\ref{lem:first_batch}, we have \(\hat{T_j} \le \frac{\eta_1}{100d}\cdot\frac{2}{d} = \frac{2\,\eta_1}{100\,d^2}\).
 
Combining both bounds with \eqref{eq:gap_identity}, on the event $A$, the empirical attention gap is bounded by:
\begin{equation}\label{eq:gap}
    \hat\Delta \;:=\; \hat\alpha_s' - \max_{j \neq s}\hat\alpha_j'\;\geq\; \frac{\eta_1\eta_2}{20\,d^2}.
\end{equation}

\textbf{Third phase.}
Conditioned on event \(A\) happens with the first batch, we fix \((\hat{\attn}', \hat{\readout}')\). All probabilities below are over the fresh batch $(w^{(k)}, y^{(k)})_{k=1}^{N_3}$, which is independent of $(\hat{\attn}', \hat{\readout}')$. From here on, $\boldsymbol\gamma(w)$ denotes the attention weights under $\hat{\attn}'$, $\boldsymbol\gamma^{(k)} := \boldsymbol\gamma(w^{(k)})$, and we denote the expected as well as the empirical hessians of \(\readout\) as \(H := \E\cb{\boldsymbol\gamma\boldsymbol\gamma^\top}\) and \(\hat H := \frac{1}{N_3} \sum_{k=1}^{N_3} \boldsymbol\gamma^{(k)} \cp{\boldsymbol\gamma^{(k)}}^\top\).
We write \(\mathcal L(\readout) := \mathcal L_s(\hat{\attn}', \readout)\) and \(\hat{\mathcal L}(\readout)\) for the empirical risk on the fresh batch. 
In contrast to the third phase of population GD, we do not analyze the convergence of $\readout$ to the risk minimizer. Instead, we compare the iterates with the fixed reference point $z := 2 e_s$.
 
We first prove that the reference point has negligible risk. Let \(r(w) := f_{\hat{\attn}', z}(w) - y(w) = 2\gamma_s(w) - 1 - y(w)\). For negative samples, \(r(w) = 0 -1 + 1 = 0\). For positive samples, when \(|w| = m\) and \(C_s(w) = n \geq 1\), we have
\[
    |r(w)| = 2 - 2\gamma_s(w) = \frac{2\cdot\sum_{j \neq s} C_j(w)\exp(\hat\alpha_j')}{n\exp(\hat\alpha_s') + \sum_{j \neq s} C_j(w)\exp(\hat\alpha_j')} \le 2(m-1)\exp(-\hat\Delta),
\]
Note that this bound does not depend on how the attention is distributed among the non-target symbols. Define the random variable \(\kappa := r(w)^2\) and \(\hat\kappa := \frac{1}{N_3}\sum_{k=1}^{N_3} r(w^{(k)})^2\). Since \(\E\cb{(m-1)^2} = (d-2)(2d-3) \leq 2d^2\), we can bound the expectation of \(\kappa\) by:
\begin{equation}\label{eq:kappa}
    \E[\kappa] \le 8 d^2 \exp(-2\hat\Delta),
\end{equation}
and by Markov's inequality, we have \(\hat\kappa \leq d^{c}\E[\kappa]\) with probability
at least $1 - d^{-c}$.
 
Now we compare population and empirical risk through \(z\). For every \(\readout \in \mathbb R^d\) and every string \(w\), we have \(f_{\hat{\attn}', \readout}(w) - y(w) = (\readout - z)^\top \boldsymbol\gamma(w) + r(w)\), and the loss can be bounded from above and below by:
\begin{align}
    & \mathcal L(\readout) \le (\readout - z)^\top H (\readout - z) + \E[\kappa], \label{eq:pop_upper} \\
    & (\readout - z)^\top \hat H (\readout - z) \le 4\hat{\mathcal L}(\readout) + 2\hat\kappa; \label{eq:emp_lower}
\end{align}
from the inequalities \((a+b)^2 \leq 2a^2 + 2b^2\) and \((a+b)^2 \geq \tfrac12 a^2 - b^2\). In particular, \(\hat{\mathcal L}(z) = \tfrac12 \hat\kappa\).
 
The empirical risk $\hat{\mathcal L}$ is a convex quadratic with Hessian $\hat H \preceq I$, so it is $1$-smooth. Let $\readout(0) = \hat{\readout}'$ and let $\readout(t)$ denote the GD iterates with $\eta_3 = 1$. By the descent lemma and convexity, for every $t \geq 0$,
\[
    \|\readout(t+1) - z\|^2 \le \|\readout(t) - z\|^2 - 2\cp{\hat{\mathcal L}(\readout(t+1)) - \hat{\mathcal L}(z)},
\]
and $\hat{\mathcal L}(\readout(t))$ is non-increasing in $t$. Summing over $t = 0, \dots, T-1$ yields
\begin{equation}\label{eq:gd}
    \hat{\mathcal L}(\readout(T)) \le \hat{\mathcal L}(z) + \frac{\|\readout(0) - z\|^2}{2T}, \; \|\readout(T) - z\|^2 \le \|\readout(0) - z\|^2 + 2T\hat{\mathcal L}(z).
\end{equation}
By \eqref{eq:beta_s} and \eqref{eq:beta_spread}, \(0 \leq \hat\beta_s' < 2\) and \(\sum_{j \neq s}(\hat\beta_j')^2 \leq (d-1)\frac{1.03\,\eta_1^2}{d^2}\le 2\), hence $\|\readout(0) - z\|^2 \leq 6$. Together with \(\hat{\mathcal L}(z) = \tfrac12\hat\kappa\), \eqref{eq:gd} gives
\begin{equation}\label{eq:gd2}
    \hat{\mathcal L}(\readout(T)) \leq \frac{\hat\kappa}{2} + \frac{3}{T}, \; \|\readout(T) - z\|^2 \leq 6 + T\hat\kappa .
\end{equation}

We next consider the concentration on the hessian with the following lemma:
\begin{lemma}\label{lem:one_sided}
    Let $\lambda := \frac{16(c+1)\ln d}{N_3}$ and assume $\lambda \leq 1$.
    Then, with probability at least $1 - d^{-c}$,
    \[
        H \;\preceq\; 2\hat H + \lambda I .
    \]
\end{lemma}
\begin{proof}
    Let $H_\lambda := H + \lambda I$ and define the independent positive semidefinite matrices
    \[
        Y_k := H_\lambda^{-1/2} \boldsymbol\gamma^{(k)}\cp{\boldsymbol\gamma^{(k)}}^\top H_\lambda^{-1/2} + \lambda H_\lambda^{-1}, \;k  \in [N_3].
    \]
    Then $\E[Y_k] = I$, so $\lambda_{\min}\cp{\sum_k \E[Y_k]} = N_3$, and $\lambda_{\max}(Y_k) \leq \frac{\|\boldsymbol\gamma^{(k)}\|_2^2}{\lambda} + 1 \leq \frac{2}{\lambda}$. By the matrix Chernoff inequality~\citep{Tropp15}, we have:
    \[
        \prob\cp{\lambda_{\min}\cp{\sum_{k=1}^{N_3} Y_k} \leq \frac{N_3}{2}} \le d\exp\cp{-\frac{N_3}{8}\cdot\frac{\lambda}{2}} = d^{-c}.
    \]
    On the complementary event, $H_\lambda^{-1/2}(\hat H + \lambda I)H_\lambda^{-1/2} \succeq \tfrac12 I$, that is, $\hat H + \lambda I \succeq \tfrac12 (H + \lambda I)$, proving the claim.
\end{proof}
 
On the events of Step~1 and Lemma~\ref{lem:one_sided}, which hold simultaneously with probability at least $1 - 2d^{-c}$, we combine \eqref{eq:pop_upper}, Lemma~\ref{lem:one_sided}, \eqref{eq:emp_lower} and \eqref{eq:gd2}:
\begin{align*}
    \mathcal L(\readout(T)) &\leq (\readout(T) - z)^\top H (\readout(T) - z) + \E[\kappa] \le 2(\readout(T) - z)^\top \hat H (\readout(T)- z) + \lambda\|\readout(T) - z\|^2 + \E[\kappa] \\
    &\leq 8\hat{\mathcal L}(\readout(T)) + 4\hat\kappa + \lambda\|\readout(T) - z\|^2 + \E[\kappa] \le \frac{24}{T} + 6\lambda + (8 + \lambda T)\hat\kappa + \E[\kappa] \le \frac{24}{T} + 6\lambda + (9 + T)d^{c}\E[\kappa] .
\end{align*}
We bound the three terms by $\varepsilon/3$ each. The first term, we can set the training step \(T = \lceil 72/\varepsilon \rceil\); for the second term, we set \(N_3 \geq \frac{288(c+1)\ln d}{\varepsilon}\), then \(6\lambda \leq \frac{\varepsilon}{3}\). For the third term, since $9 + T \leq \frac{82}{\varepsilon}$, by \eqref{eq:kappa} we know $(9 + T)\,d^{c}\E[\kappa] \leq \frac{656\, d^{c+2}}{\varepsilon} \exp(-2\hat\Delta) \leq \frac{\varepsilon}{3}$ if $\hat\Delta \geq \ln\frac{45\, d^{(c+2)/2}}{\varepsilon}$. By \eqref{eq:gap}, this will hold if $\eta_2 \geq \frac{20 d^2}{\eta_1} \ln\frac{45\, d^{(c+2)/2}}{\varepsilon}$, which is implied by $\eta_2 \geq \frac{C d^2}{\eta_1}\ln\frac{d}{\varepsilon}$ for a suitable constant $C$.
Hence $\mathcal L_s(\hat{\attn}', \hat{\readout}'') \leq \varepsilon$ with $\hat{\readout}'' := \readout(T)$. A union bound over the event $A$ and the two events above gives a total probability at least $1 - 3d^{-c}$. 
\end{proof}

\section{Additional Experiments}

\subsection{Experiment Setup}\label{app:experiment_setup}
To produce the experiment results, we adopt the following initialization, losses, step sizes, and step counts:

\paragraph{Models.} For \texttt{even pairs} and \texttt{SLR}, we use the one-layer, single-head models of~\citet{HuangL025} and~\citet{MarionBBB25}, with parameters $(W \in \mathbb{R}^{D \times D}, u \in \mathbb{R}^{D})$ and $(k \in \mathbb{R}^{D}, v \in \mathbb{R}^{D})$, respectively, where $D$ is the embedding dimension. In both cases the first parameter plays the role of our attention weights and the second that of our output weights. For \texttt{ICL bigram}, we use a single attention layer with one-hot token embeddings and parameters $A, B \in \mathbb{R}^{d \times d}$. Given the query token $q$, position $t$ receives the attention logit $A_{q, x_{t-1}}$, which depends on its preceding token $x_{t-1}$ through the fixed first layer. The model outputs the logits \( \sum_{t} \mathrm{softmax}_t\big(A_{q, x_{t-1}}\big)\, B_{:, x_t} \in \mathbb{R}^{d}\)  over the next token and is trained with the cross-entropy loss. Here, $A$ and $B$ play the roles of our attention and output weights, respectively.
\paragraph{Initialization.}
Table~\ref{tab:init_config} lists the initializations. \texttt{AND} differs in three respects. First, we draw $\attn^{(1)}, \attn^{(2)}, \readout^{(1)}, \readout^{(2)}$ with i.i.d.\ $\mathcal{N}(0, 1)$ entries and set $\attn^{(3)} = -\attn^{(1)}$, $\attn^{(4)} = -\attn^{(2)}$, $\readout^{(3)} = -\readout^{(1)}$, $\readout^{(4)} = -\readout^{(2)}$, which prevents head collapse. Second, in the plans A$\to$O, O\(\to\)A, and O$\to$A$\to$O, the attention phase consists of $1250$ GD steps instead of one, since a single step did not suffice for the heads to attend to different targets. 


\begin{table}[ht]
    \centering
    \small
    \setlength{\tabcolsep}{4pt}
    \caption{Initializations. The fixed initialization is used only for O$\to$A$\to$O; Gaussian entries are i.i.d.}
    \label{tab:init_config}
    \begin{tabular}{@{}lll@{}}
    \toprule
    Task & Fixed & Random \\
    \midrule
    Occurrence, \texttt{OR}
        & $\attn = \readout = \mathbf{0}_d$
        & $\attn, \readout \sim \mathcal{N}(0, 0.1^2)$ \\
    \texttt{AND}
        & $\attn^{(h)} = \readout^{(h)} = \mathbf{0}_d$ for \(h\in[4]\)
        & see text \\
    \texttt{Even pairs}
        & $W = \mathbf{0}_{D \times D},\ u = \mathbf{0}_D$
        & $W, u \sim \mathcal{N}(0, 0.1^2)$ \\
        \texttt{ICL-bigram}
        & $A = B = \mathbf{0}_d$
        & $A, B \sim \mathcal{N}(0, 0.1^2)$ \\
    \texttt{SLR}
        & $k = v = \mathbf{0}_d$
        & $k, v \sim \mathcal{N}(0, 0.1^2/d)$ \\
    \bottomrule
    \end{tabular}
\end{table}

\paragraph{Step sizes and counts.}
Every plan is built from the same three ingredients: a base step size $\eta$ for multi-step phases and for joint GD, and the step sizes $\eta_{\mathrm{O}}$ and $\eta_{\mathrm{A}}$ of the single output and attention steps. Table~\ref{tab:schedules} gives the schedule of each plan and Table~\ref{tab:step_sizes} gives the values per task.

For squared-loss experiments, we minimize and report MSE without the factor $1/2$ used in the theoretical analysis. An experimental learning rate $\eta$ therefore corresponds to $2\eta$ under the theoretical loss normalization.

\begin{table}[ht]
    \small
    \begin{minipage}[t]{0.58\textwidth}
        \centering
        \setlength{\tabcolsep}{4pt}
        \caption{Schedule of each training plan.
        ``A: $\eta_{\mathrm{A}} \times 1$'' means one GD step on the
        attention weights with step size $\eta_{\mathrm{A}}$;
        ``O: $\eta \times 2500$'' means $2500$ GD steps on the output
        weights with step size $\eta$.}
        \label{tab:schedules}
        \begin{tabular}{@{}llll@{}}
        \toprule
        Plan & Phase 1 & Phase 2 & Phase 3 \\
        \midrule
        A only                & A: $\eta \times 2502$ & -- & -- \\
        O only                & O: $\eta \times 2502$ & -- & -- \\
        A(1x)$\to$O           & A: $\eta_{\mathrm{A}} \times 1$
                              & O: $\eta \times 2500$ & -- \\
        O(1x)$\to$A(1x)       & O: $\eta_{\mathrm{O}} \times 1$
                              & A: $\eta_{\mathrm{A}} \times 1$ & -- \\
        O(1x)$\to$A(1x)$\to$O & O: $\eta_{\mathrm{O}} \times 1$
                              & A: $\eta_{\mathrm{A}} \times 1$
                              & O: $\eta \times 2500$ \\
        Joint GD              & \multicolumn{3}{l}{A and O jointly:
                                $\eta \times 2502$} \\
        \bottomrule
        \end{tabular}
    \end{minipage}%
    \hfill
    \begin{minipage}[t]{0.38\textwidth}
        \centering
        \setlength{\tabcolsep}{5pt}
        \caption{Step sizes per task. For \texttt{AND}, the attention
        phase uses multiple steps with the base step size $\eta$ (see
        text).}
        \label{tab:step_sizes}
        \begin{tabular}{@{}lccc@{}}
        \toprule
        Task & $\eta$ & $\eta_{\mathrm{O}}$ & $\eta_{\mathrm{A}}$ \\
        \midrule
        Occurrence           & $1$    & $50/3$ & $500$ \\
        \texttt{OR}          & $1$    & $12.5$ & $500$ \\
        \texttt{AND}         & $0.25$ & $8$    & -- \\
        \texttt{Even pairs}  & $1$    & $100$  & $1000$ \\
        \texttt{SLR}         & $0.03$ & $3$    & $100$ \\
        \texttt{ICL bigram} & $50$ & $625$ & $25{,}000$ \\
        \bottomrule
        \end{tabular}
    \end{minipage}
\end{table}

For \texttt{AND}, every single attention step is replaced by a multi-step
phase with step size $\eta$: A$\to$O uses $1250$ attention steps followed
by $1250$ output steps, O(1x)$\to$A uses one output step followed by $1250$
attention steps, and O(1x)$\to$A$\to$O uses 1250 attention steps in the second phase and 1250 output steps in the third.

\subsection{Attention and Output Gaps}\label{app:gaps}


\textbf{\texttt{OR}.} With target set $T$ ($|T| = 2$), the attention gap is $\mathbb{E}_+\big[\sum_{t\in T}\gamma_t(w) - \max_{i\notin T}\gamma_i(w)\big]$. The output gap is $\big(\tfrac{1}{|T|}\sum_{t\in T}\beta_t - \max_{j\notin T}\beta_j\big)/2$.

\textbf{\texttt{AND}.} For each head $h$, let $t_h = \arg\max_{t\in T}\alpha^{(h)}_t$ be the target that head favours. Its attention gap is $\mathbb{E}_+\big[\gamma^{(h)}_{t_h}(w) - \max_{i\neq t_h}\gamma^{(h)}_i(w)\big]$. Heads are ranked by their final gap, separately in every run. 

\textbf{\texttt{even pairs}.} Let $z(X) = \sum_{l}\phi_l(X)\,x_l$ be the attention output, where $\phi$ are the attention weights under $W$. The attention measure is the linear separability of $\{(z(X), y)\}$ over all strings. It is the accuracy of the max-margin linear classifier $w^*$ through the origin, computed with a squared-hinge SVM with $C = 10^3$. The output gap is $\cos(u, w^*)$. Since $w^*$ is recomputed from the current attention, this gap can change while $u$ is frozen.

\textbf{\texttt{ICL bigram}.} For a sequence with query $q$, let $t^*$ be the position after the query key, i.e., the position with $x_{t^*-1} = q$. The attention gap is the average over the training sequences of $\mathrm{softmax}_{t^*} - \max_{t\neq t^*}\mathrm{softmax}_{t}$, with weights $\mathrm{softmax}_t(A_{q,x_{t-1}})$. The output gap is $\cos(\tilde B, \tilde I)$, using the Frobenius inner product. Here $\tilde B$ and $\tilde I = I - \tfrac{1}{d}\mathbf{1}\mathbf{1}^\top$ are $B$ and the identity with each column centered over the classes.

\textbf{\texttt{SLR}.} With $J$ the relevant location of a sample, the attention gap is the average over the training samples of $\mathrm{softmax}(Xk)_J - \max_{l\neq J}\mathrm{softmax}(Xk)_l$. The output gap is $\cos(v, v^*)$ with the ground-truth direction $v^*$.

\paragraph{Compute resources.} All experiments were run on an Apple MacBook Air with an M2 chip (8 CPU cores) and 8 GB of unified memory, running macOS 14.4. No GPU was used; all training is full-batch GD on the CPU. 

\subsection{Learning progress for \texttt{AND} task across all heads.}\label{app:and_all_heads}
\begin{figure*}[h]
    \centering
    \includegraphics[width=\linewidth]{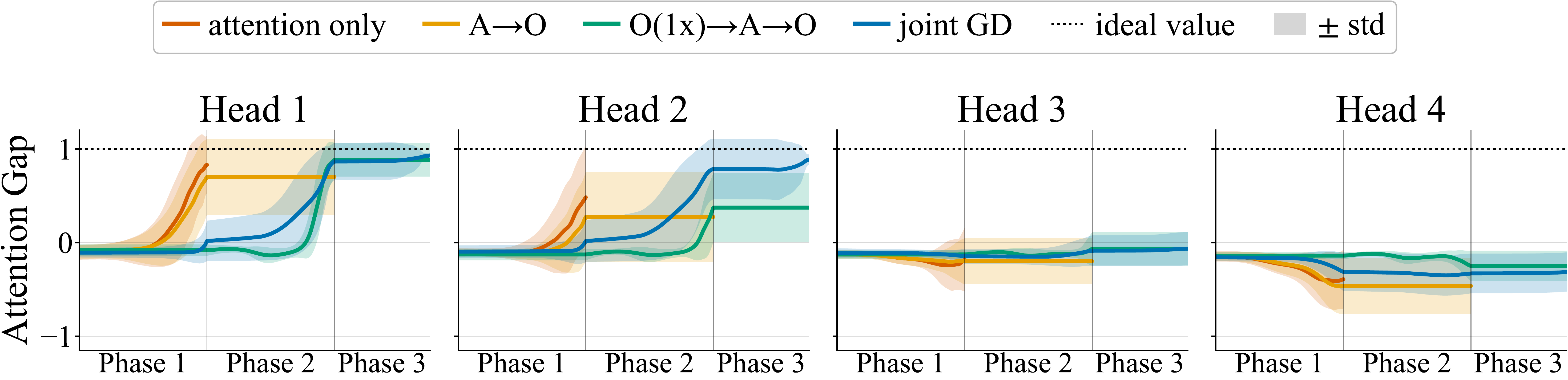}
    \caption{The learning progress of the \texttt{AND} task for each head, with average trajectory and std (shade). The heads are ranked by the gap they reach at the end, per head and per run.}\label{fig:trajectory_all_heads}
\end{figure*}
Figure~\ref{fig:trajectory_all_heads} shows the evolution of the per-head attention gap for the \texttt{AND} task. Under all training plans, at most two of the four heads develop a positive gap, with the gaps of Heads 3 and 4 staying near zero or becoming negative. The plans differ mostly in the second head. Under joint GD, both Head 1 and Head 2 reach a gap close to the ideal value, so the two heads attend to the two target symbols. Under the three-phase plan, Head 1 saturates closely in the same way, but the gap of Head~2 reaches only about $0.4$ on average, with a large variance across runs: in many runs, the second target symbol receives little more attention than the non-target symbols. A$\to$O and attention-only training show the same pattern, with one head close to saturation and a weaker second head. This explains why the phased plans do not match joint GD on \texttt{AND} in Table~\ref{tab:complete_experiments}: recognizing the language requires attending to both target symbols, and only joint GD reliably achieves this.

\subsection{Complete results with standard deviation.}\label{app:table_with_std}
\begin{table}[ht]
\centering
\caption{Train loss and test accuracy (loss for \texttt{SLR}) as mean over $50$ seeds, with the standard deviation in parentheses. Values \(\le10^{-3}\) are shown in upper bounds; those that match or surpass joint GD are bold. The red entries are deterministic runs because the training and test sets are the entire instance space.}
\label{tab:complete_experiments_std}
\setlength{\tabcolsep}{4pt}
\resizebox{\textwidth}{!}{%
\begin{tabular}{@{}llcccccccccccc@{}}
\toprule
\multicolumn{2}{@{}c}{\multirow{2}{*}[-1.2ex]{\textbf{Training plan}}} & \multicolumn{2}{c}{\textbf{Occurrence}} & \multicolumn{2}{c}{\textbf{OR}} & \multicolumn{2}{c}{\textbf{AND}} & \multicolumn{2}{c}{\textbf{Even-Pair}} & \multicolumn{2}{c}{\textbf{ICL bigram}} & \multicolumn{2}{c}{\textbf{SLR}} \\
\cmidrule(lr){3-4}\cmidrule(lr){5-6}\cmidrule(lr){7-8}\cmidrule(lr){9-10}\cmidrule(lr){11-12}\cmidrule(lr){13-14}
 & & \hd{Train}{Loss} & \hd{Test}{Acc.} & \hd{Train}{Loss} & \hd{Test}{Acc.} & \hd{Train}{Loss} & \hd{Test}{Acc.} & \hd{Train}{Loss} & \hd{Test}{Acc.} & \hd{Train}{Loss} & \hd{Test}{Acc.} & \hd{Train}{Loss} & \hd{Test}{Loss} \\
\midrule
Coupled with rand init. & Joint GD & \(10^{-3}\) & \(1.00\) & \(10^{-3}\) & \(1.00\) & \(0.04\) & \(0.99\) & \(0.01\) & \(1.00\) & \(10^{-3}\) & \(1.00\) & \(0.01\) & \(0.01\) \\
 & & \sd{10^{-6}} & \sd{0.00} & \sd{10^{-5}} & \sd{0.00} & \sd{0.06} & \sd{10^{-3}} & \sd{10^{-4}} & \sd{0.00} & \sd{10^{-7}} & \sd{0.00} & \sd{10^{-3}} & \sd{10^{-3}} \\
\midrule
\multirow{10}{*}{\shortstack[l]{Decoupled 1/2/3-phase\\with random init.}}
 & Attention (A) only & \(1.64\) & \(0.50\) & \(1.67\) & \(0.50\) & \(0.24\) & \good{0.94} & \(0.69\) & \(0.51\) & \(3.90\) & \(0.02\) & \(0.50\) & \(0.50\) \\
 & & \sd{0.08} & \sd{0.01} & \sd{0.06} & \sd{0.01} & \sd{0.18} & \sd{0.07} & \sd{10^{-3}} & \sd{0.02} & \sd{0.00} & \sd{0.01} & \sd{0.02} & \sd{0.02} \\\addlinespace[2pt]
 & Output (O) only & \(0.74\) & \(0.86\) & \(0.68\) & \(0.85\) & \(0.79\) & \(0.75\) & \(0.63\) & \(0.50\) & \(2.70\) & \(0.16\) & \(0.48\) & \(0.48\) \\
 & & \sd{0.05} & \sd{0.01} & \sd{0.04} & \sd{0.02} & \sd{0.03} & \sd{0.02} & \sd{0.00} & \sd{0.00} & \sd{0.04} & \sd{0.01} & \sd{0.01} & \sd{0.01} \\\addlinespace[2pt]
 & A(1x)$\to$O & \(0.72\) & \(0.76\) & \(0.69\) & \(0.75\) & \(0.30\) & \(0.91\) & \(0.62\) & \(0.50\) & \(2.58\) & \(0.26\) & \(0.41\) & \(0.43\) \\
 & & \sd{0.26} & \sd{0.13} & \sd{0.21} & \sd{0.11} & \sd{0.19} & \sd{0.09} & \sd{0.02} & \sd{0.11} & \sd{0.86} & \sd{0.22} & \sd{0.14} & \sd{0.14} \\\addlinespace[2pt]
 & O(1x)$\to$A(1x) & \(0.47\) & \good{1.00} & \(0.64\) & \(0.78\) & \(0.52\) & \(0.87\) & \(0.52\) & \(0.50\) & \(3.20\) & \good{1.00} & \(0.28\) & \(0.28\) \\
 & & \sd{0.07} & \sd{0.00} & \sd{0.07} & \sd{0.19} & \sd{0.19} & \sd{0.07} & \sd{10^{-3}} & \sd{10^{-3}} & \sd{0.02} & \sd{0.00} & \sd{0.03} & \sd{0.03} \\\addlinespace[2pt]
 & O(1x)$\to$A(1x)$\to$O & \good{10^{-5}} & \good{1.00} & \good{10^{-5}} & \good{1.00} & \(0.27\) & \good{0.93} & \good{0.08} & \good{1.00} & \good{10^{-3}} & \good{1.00} & \good{0.01} & \good{0.01} \\
 & & \sd{10^{-6}} & \sd{0.00} & \sd{10^{-5}} & \sd{0.00} & \sd{0.13} & \sd{0.06} & \sd{10^{-3}} & \sd{0.00} & \sd{10^{-8}} & \sd{0.00} & \sd{10^{-3}} & \sd{10^{-3}} \\
\midrule
Sec.~\ref{sec:positive_result}: 3-phase, zero init. & O(1x)$\to$A(1x)$\to$O & \good{10^{-8}} & \good{1.00} & \good{10^{-7}} & \good{1.00} & \(0.90\) & \(0.68\) & \textcolor{red}{\good{0.08}} & \textcolor{red}{\good{1.00}} & \good{10^{-3}} & \good{1.00} & \good{0.01} & \good{0.01} \\
 & & \sd{10^{-9}} & \sd{0.00} & \sd{10^{-6}} & \sd{0.00} & \sd{0.02} & \sd{0.02} & \textcolor{red}{\sd{0.00}} & \textcolor{red}{\sd{0.00}} & \sd{10^{-8}} & \sd{0.00} & \sd{10^{-3}} & \sd{10^{-3}} \\
\bottomrule
\end{tabular}%
}
\end{table}

\subsection{Learning progress on occurrence and \texttt{OR}} \label{app:learning_progress_occur_or}
Figure~\ref{fig:trajectory_occur_or} shows the trajectories on occurrence and \texttt{OR} for all training plans. Both tasks behave as in Figure~\ref{fig:gd_comparison}: under joint GD, the output gap grows first, overshoots while the attention saturates, and then settles at its ideal value, and the three-phase plan reaches the same final gaps. When the attention is trained first or alone, its gap moves in the wrong direction, and the output weights then grow far beyond their ideal value. 

\begin{figure}[t]
\centering 
\includegraphics[width=0.7\linewidth]{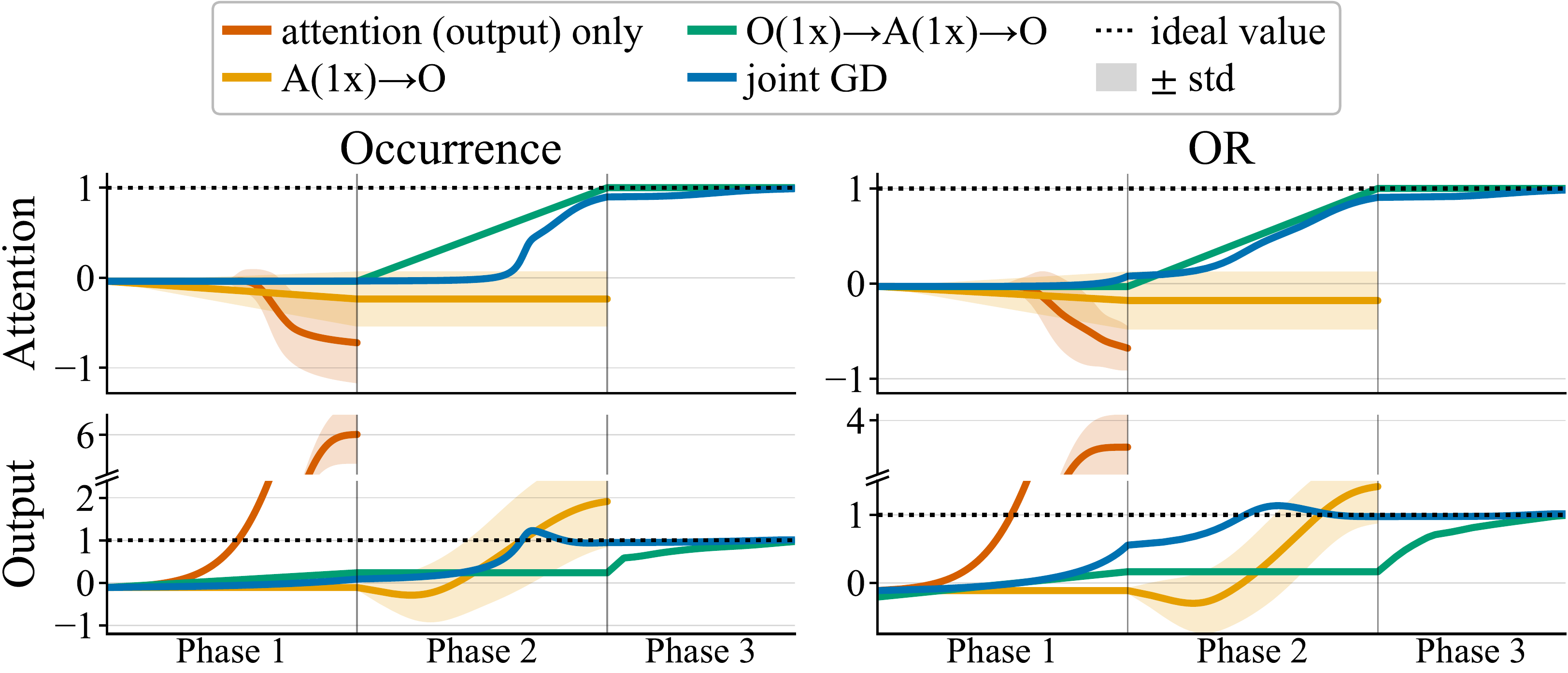}
\caption{\textbf{Attention and output gaps on occurrence and \texttt{OR}.} Notation and phases are as in Figure~\ref{fig:trajectory_complete}.} \label{fig:trajectory_occur_or}
\end{figure}

\end{document}